\documentclass[journal]{IEEEtran}

\usepackage{amsmath,amssymb,amsfonts}
\usepackage{algorithm}
\usepackage[noend]{algpseudocode}
\usepackage{graphicx}
\usepackage{xcolor}
\usepackage{textcomp}
\usepackage{booktabs}
\usepackage{float}
\usepackage{bm}
\usepackage[colorlinks,linkcolor=blue,citecolor=blue]{hyperref}
\usepackage{multirow,makecell}
\usepackage{tikz}
\usepackage{pgfplots}
\pgfplotsset{compat=1.18}
\usetikzlibrary{calc}
\usepackage{amsthm}
\usepackage{cite}

\newtheorem{proposition}{Proposition}
\newtheorem{lemma}{Lemma}
\newtheorem{corollary}{Corollary}

\begin{document}

\title{Ghost Attractor Networks: Basin-Structured Dynamical Decoders for Closed-Loop Sequential Generation}

\author{Tianyu~Wang,
        Ying~Wang,
        Zhihao~Liu,
        Xi~Vincent~Wang,
        and~Lihui~Wang%
\thanks{This work has been submitted to the IEEE for possible publication. Copyright may be transferred without notice, after which this version may no longer be accessible.}%
\thanks{This research work was supported by the Swedish research centre of eXcellence in PRoduction RESearch (XPRES).}%
\thanks{Tianyu Wang, Zhihao Liu, Xi Vincent Wang, and Lihui Wang are with the Department of Production Engineering, KTH Royal Institute of Technology, Stockholm, Sweden (corresponding author: Tianyu Wang, e-mail: tianyuwa@kth.se).}%
\thanks{Ying Wang is with the Department of Decision and Control Systems, KTH Royal Institute of Technology, Stockholm, Sweden.}}

\markboth{Preprint. This work has been submitted to the IEEE for possible publication.}%
{Wang \MakeLowercase{\textit{et al.}}: Ghost Attractor Networks}

\maketitle

\begin{abstract}
Sequential output generation with large-scale Transformer and diffusion decoders pays a memory cost that grows with sequence length, plus iterative per-step computation. Replacing them with small feed-forward decoders restores efficiency but produces unstructured latent representations that limit closed-loop control: phase-conditioned action generation and cross-step latent carry-over both require a latent geometry with stable basins. This article proposes Ghost Attractor Networks, a theoretically derived dynamical decoder whose latent evolves under a learned potential with drift and produces a basin-attractor structure by construction. Three desiderata (multi-modality, decoder-level single-pass switching, and constant memory) motivate the potential-drift form, and mode transitions arise as saddle-node bifurcations with ghost-attractor escape. A hierarchical phase-space decomposition separates first-order basin convergence from second-order proprioceptive refinement. Empirically, a Ghost trained end-to-end with a behavioral-cloning and contrastive objective exhibits the predicted gradient-flow contraction in its potential, with the gradient norm decaying by 67 percent across five integration steps on 1430 held-out samples. Ghost is evaluated as a robotic action decoder. A 2.3-million-parameter Ghost matches the offline accuracy of a 1.07-billion-parameter Diffusion Transformer at 462 times fewer parameters and 32 times lower latency, and beats five alternative 2M-parameter decoders (MLP, Neural ODE, CVAE, Transformer, 1-step Diffusion) on offline mean squared error by 5.9 to 29 percent. On the LIBERO-10 closed-loop benchmark, phase conditioning on Ghost's basin-structured latent yields a 13.5 percentage-point success-rate gain over a feed-forward MLP baseline, and persistent-latent ensembling reaches a 95.7 percent final success rate.
\end{abstract}

\begin{IEEEkeywords}
Dynamical-system neural networks, multi-task learning, neural attractor dynamics, neural ODE, zero-weight adaptation.
\end{IEEEkeywords}

\section{Introduction}
\label{sec:intro}

\IEEEPARstart{L}{arge-scale} neural networks for sequential output generation have advanced rapidly in robotics, autonomous decision-making, and embodied AI~\cite{brohan2023rt2,bjorck2025groot,kim2024openvla}. Yet these architectures share a structural limitation: temporal context scales linearly with the interaction horizon, and mode transitions require accumulating enough new tokens to overwrite prior context. For a network generating outputs at 100~Hz, switching behavioral mode is not instantaneous; it takes tens to hundreds of timesteps of context accumulation, during which the network produces ambiguous or incorrect outputs.

This limitation is rooted in the dominant architectural paradigm. Transformer key-value caches~\cite{vaswani2017attention}, iterative diffusion decoders~\cite{chi2023diffusion}, and in-context learning histories~\cite{laskin2022incontext} all encode temporal context as growing token sequences with $\mathcal{O}(t)$ memory cost. Per-step inference therefore scales with both sequence length and decoder size: a billion-parameter Diffusion Transformer head requires several denoising iterations per control step, and a Transformer decoder with full attention reads its entire cached history on every forward pass. At deployment frequencies of 50--100~Hz, this combined cost consumes most of the per-step time budget, leaving little room for perception, control, or safety logic. The reliance on context accumulation also delays behavioral transitions, since the decoder must observe enough new tokens for its output to reflect a new mode, which makes instant mode switching structurally infeasible.

A straightforward remedy is to replace the large decoder with a small feed-forward network (MLP, Neural ODE, or single-step diffusion). This restores efficiency but introduces a different deficit specific to closed-loop control. Sequential decoding with modal behavioral structure imposes two coupled requirements on a deployed policy: (i)~specialization across qualitatively distinct sub-phases via an external phase-conditioning signal, and (ii)~integration across timesteps because each observation is partial, typically by carrying the decoder's own latent state $z_t$ into the next forward pass. Both mechanisms presuppose that the latent space is geometrically organized into stable basins. Phase conditioning then routes the trajectory between qualitatively different basins rather than perturbing a generic feature vector, and persistent-$z$ carry-over remains stable only when each basin has an equilibrium for the latent to settle into between steps. A plain feed-forward decoder provides no such geometry. Under matched parameter budget, phase conditioning on a stateless decoder collapses closed-loop performance, while the same mechanism on a basin-structured latent yields a substantial gain (quantified in Section~\ref{sec:results}).

Biological neural circuits show that the two deficits above need not be traded against each other. In cortex and basal ganglia, behavioral repertoires are encoded as discrete dynamical regimes (basins) of the neural state space, with transitions mediated by continuous trajectories that pass through unstable saddle regions before settling into a new regime~\cite{vyas2020computation,churchland2012neural,sussillo2013opening}. The representation is both efficient (no growing token cache, the latent state has fixed dimension) and geometrically structured, with basins for distinct modes that support stable carry-over and rapid transitions. This biological example motivates an architectural principle: encode behavioral modes as basins of a learned potential, and let mode transitions arise as continuous trajectories driven by external context, rather than from token accumulation (the large-decoder failure) or unstructured feature mappings (the small-decoder failure).

This article introduces Ghost Attractor Networks, a neural architecture that addresses both the efficiency and the structural-deficit problems above. The central mechanism is the ghost attractor~\cite{strogatz2015nonlinear}: a region of state space where the vector field nearly vanishes, transiently trapping trajectories before releasing them along escape channels into different behavioral basins. The network learns a potential landscape whose topology encodes distinct behavioral modes as attractor basins, and uses ghost-attractor escape dynamics to mediate transitions. A context change therefore produces a switched action within a single decoder forward pass, while the basin geometry that closed-loop phase conditioning and persistent-$z$ carry-over rely on is inherited from the architecture itself. The presence of this attractor mechanism is verified directly in a Ghost trained end-to-end on the multi-task data: the gradient norm $\|\nabla_z U\|$ contracts by $67\%$ across the five integration steps (Fig.~\ref{fig:attractor_emergence}), matching the gradient-flow signature implied by the architectural form. This addresses three open challenges in sequential decoding: horizon-independent memory cost, decoder-level single-pass behavioral switching, and structured multi-task compression that supports closed-loop deployment. Each is discussed in detail through the contributions below.

The contributions of this article are as follows:

\begin{enumerate}
\item[1)] Ghost Attractor Networks, a constant-memory dynamical decoder, are proposed. The decoder encodes behavioral modes as basins of a learned potential and switches between them via ghost-attractor escape. The resulting latent geometry supports phase conditioning and persistent-$z$ carry-over. The architecture uses a hierarchical two-phase integration within the same potential-drift framework: first-order basin convergence, followed by second-order proprioceptive refinement near basin minima.
\item[2)] This form is derived from three desiderata (multi-modality, single-pass switching at the decoder level, and constant memory) and supported by formal analysis: dwell-time scaling (Lemma~\ref{lem:dwell}), ghost-mediated switching (Proposition~\ref{prop:switching}), combinatorial expressivity (Corollary~\ref{cor:combinatorial}), and a local Lyapunov stability argument.
\item[3)] The architecture is empirically validated in a trained model. The predicted gradient-flow contraction is directly observed in a trained Ghost (Fig.~\ref{fig:attractor_emergence}). A 2.3M-parameter Ghost decoder delivers order-of-magnitude reductions in parameter count and inference latency relative to a 1.07B Diffusion Transformer, lower offline imitation error than five same-budget decoder alternatives, and a substantially higher closed-loop success rate than a stateless baseline under matched parameter budget. Detailed comparisons, ablations, and external benchmark anchoring are deferred to Section~\ref{sec:results}.
\end{enumerate}

\section{Related Work}
\label{sec:related}

\subsection{Neural Dynamical Systems}
Neural ODEs~\cite{chen2018neuralode} parameterize continuous-time dynamics with neural networks, but produce smooth trajectories with no mechanism for abrupt mode transitions. Modern Hopfield networks~\cite{ramsauer2021hopfield} connect attractor-based memory to Transformer attention but focus on pattern retrieval rather than sequential output generation. Hamiltonian Neural Networks~\cite{greydanus2019hamiltonian} impose energy-conservation structure for physical-dynamics learning, again with no mechanism for behavioral switching. Ghost attractors near saddle-node bifurcations have been studied in physical systems~\cite{schneider2020ghost,strogatz2015nonlinear}, and Durstewitz et al.~\cite{durstewitz2025neuroscience} argue that abrupt biological learning transitions are mediated by bifurcation-induced state-space reorganizations analogous to ghost-attractor escape. The proposed work turns this principle into a learnable, end-to-end-trained neural decoder, in which the ghost mechanism is a designed architectural primitive in the latent action space rather than an emergent phenomenon.

\subsection{Adaptive Policy Architectures}
Meta-reinforcement learning achieves rapid task adaptation via learned initialization~\cite{finn2017maml}, context-conditioned policies~\cite{rothfuss2019promp}, or hypernetwork-generated parameters, but typically requires online gradient computation or episodic context accumulation. In-context RL~\cite{laskin2022incontext,lee2024dpt,duan2016rl2} enables zero-weight adaptation through attention over interaction histories; however, growing token sequences incur $\mathcal{O}(t)$ memory cost and require accumulation of sufficient new observations before behavior can change. Recurrent architectures such as GRU and state-space models maintain constant memory, but encode behavioral modes implicitly in weight matrices and require slow context integration to switch. The proposed architecture instead encodes modes as explicit basins in a learned potential, enabling single-step transitions driven by a low-dimensional modulation signal.

\subsection{Generative Sequential Decoding}
Diffusion policies~\cite{chi2023diffusion,pearce2023imitating} generate action sequences through iterative denoising. Flow matching~\cite{lipman2023flow} offers a simulation-free training alternative with straighter generation paths, and consistency models~\cite{song2023consistency,prasad2024consistency} reduce the number of denoising steps required. All such iterative generative approaches incur per-step overhead that scales with the number of iterations. RDT-1B~\cite{liu2024rdt} pushes Diffusion Transformer decoders to a billion parameters for multi-task competence, and Tan et al.~\cite{luo2024robot_tnnls} model multi-task manipulation via visuomotor latent diffusion. In contrast, the proposed dynamical decoder replaces iterative denoising with a fixed, small number of integration steps and reaches comparable accuracy at much lower parameter cost.

\subsection{Foundation Models with Sequential Decoders}
Recent vision-language-action (VLA) foundation models for robotics unify perception, language, and action generation in a single architecture trained on large-scale demonstration data. RT-1/RT-2~\cite{brohan2022rt1,brohan2023rt2} established the VLA paradigm with Transformer-based policies; GR00T N1~\cite{bjorck2025groot} extended this to humanoid robots with a dual-system design that separates a vision-language backbone from a Diffusion Transformer decoder. OpenVLA~\cite{kim2024openvla}, Octo~\cite{octo2024}, $\pi_0$~\cite{black2024pi0}, OpenVLA-OFT~\cite{kim2025openvla_oft}, and Dream-VLA~\cite{ye2025dreamvla} explore variants of the same paradigm. All these approaches rely on Transformer or diffusion decoders with $\mathcal{O}(t)$ memory growth and large per-step parameter counts. The proposed work is complementary: the pretrained backbone is kept, and only the decoder is replaced with a constant-memory dynamical alternative.

\subsection{Dynamical Movement Primitives and Attractor Landscapes}
Classical Hopfield networks~\cite{hopfield1982neural} store patterns as fixed-point attractors. Dynamical movement primitives (DMPs)~\cite{schaal2006dynamic,ijspeert2013dynamical} encode motor behaviors as stable limit cycles or point attractors, and the Stable Estimator of Dynamical Systems (SEDS)~\cite{khansari2011learning} learns globally stable nonlinear dynamics from demonstrations. All DMP and SEDS variants, however, require either a separate dynamical system per mode or an explicit discrete supervisor to select between modes; the dynamical system itself does not execute the transition. The proposed framework removes this discrete supervisor: mode switching becomes an emergent property of a single potential landscape, where modulation annihilates one basin into a ghost region and the trajectory escapes to another. The biological grounding (behaviors as distinct flow fields in neural population state space, with transitions through continuous dynamical evolution~\cite{vyas2020computation,sussillo2013opening,churchland2012neural}) further motivates this architectural principle.

\section{Theoretical Foundations}
\label{sec:theory}

The proposed architecture is derived from first principles. The requirements of multi-mode decoding with constant memory lead to potential-driven dynamics with drift on a modulated landscape, and the ghost-attractor switching mechanism emerges as a generic consequence. The modulatory context signal is denoted $\mathbf{c}_t \in \mathbb{R}^{d_c}$ throughout this section; Section~\ref{sec:method} instantiates it as a context embedding $\mathbf{e}_t$.

\subsection{From Desiderata to Potential Dynamics}

Consider a general fixed-parameter dynamical decoder. Let $\mathbf{x}_t \in \mathbb{R}^{d_x}$ denote the latent state and $\mathbf{c}_t \in \mathbb{R}^{d_c}$ a modulatory control signal that encapsulates all external inputs (observations, task context, etc.). Throughout the theory the dynamics are continuous-time, and are discretized via explicit Euler (step size $\Delta t$) at the architectural level (Section~\ref{sec:method}). The latent dynamics evolve as:
\begin{equation}
\dot{\mathbf{x}} = \mathbf{F}_{\theta}(\mathbf{x}, \mathbf{c}_t),
\label{eq:full_dynamics}
\end{equation}
where $\mathbf{F}_{\theta}$ is a neural vector field and no weight updates on $\theta$ occur during adaptation. Actions are decoded as $\mathbf{a}_t = \boldsymbol{\pi}_{\theta}(\mathbf{x}_t)$. Three practical requirements motivate a specific functional form for $\mathbf{F}_\theta$.

Three properties of the decoder are required: (i)~\emph{multi-modality}, i.e., support for $K \ge 2$ distinct stable output regimes, each producing a different action profile, which demands $K$ stable equilibria in the latent space; (ii)~\emph{instant switching}, where a change in the external signal $\mathbf{c}_t$ triggers a transition between regimes; and (iii)~\emph{constant memory}, meaning the state dimension $d_x$ is independent of the number of modes $K$ and the temporal horizon.

The derivation below shows that these requirements motivate the form
\begin{equation}
\dot{\mathbf{z}} = \underbrace{-\nabla_\mathbf{z} U_\theta(\mathbf{z}; \mathbf{c}_t)}_{\text{potential gradient}} + \underbrace{\mathbf{b}_\theta(\mathbf{z}; \mathbf{c}_t)}_{\text{drift}},
\label{eq:derived_slow}
\end{equation}
where each term is justified by a separate step:
\emph{Step~1} introduces the potential gradient $-\nabla_\mathbf{z} U_\theta$ to encode multiple stable basins (property~(i));
\emph{Step~2} shows how varying $\mathbf{c}_t$ enables switching (property~(ii)) via saddle-node bifurcation, with the ghost-attractor mechanism emerging generically;
\emph{Step~3} adds the drift $\mathbf{b}_\theta$ to break the symmetry of pure gradient flow.
Eq.~\eqref{eq:derived_slow} is refined via fast-slow decomposition in Section~\ref{sec:theory}.B and instantiated in Section~\ref{sec:method}. The potential maintains basin stability (via the Lyapunov argument in Section~\ref{sec:lyapunov}), while the drift enables task-dependent directional preferences.

\emph{Step 1: Multi-modality implies a potential.}
Property~(i) demands multiple stable equilibria. Among vector fields admitting such equilibria, gradient systems $\dot{\mathbf{z}} = -\nabla_\mathbf{z} U_\theta(\mathbf{z}; \mathbf{c}_t)$ for a scalar potential $U_\theta: \mathbb{R}^{d_z} \times \mathbb{R}^{d_c} \to \mathbb{R}$ (with $\mathbf{z}$ the basin-carrying component of $\mathbf{x}$, formalized in Section~\ref{sec:theory}.B) are distinguished by a key structural property: the potential decreases monotonically along trajectories at any fixed $\mathbf{c}_t$ ($\dot{U}_\theta = -\|\nabla_\mathbf{z} U_\theta\|^2 \le 0$ when $\mathbf{c}_t$ is held constant), which precludes limit cycles and chaotic transients (by the Poincar\'{e}--Bendixson theorem in 2D; by the gradient structure in general)~\cite{strogatz2015nonlinear}. Each local minimum of $U_\theta$ defines a stable basin, encoding a behavioral mode. Crucially, a single scalar function over $\mathbb{R}^{d_z}$ can host an arbitrary number of minima while $d_z$ (and thus $d_x$) remains fixed, satisfying property~(iii). This justifies the potential-gradient term $-\nabla_\mathbf{z} U_\theta$ in Eq.~\eqref{eq:derived_slow}.

\emph{Step 2: Switching implies modulated potential and saddle-node bifurcation.}
Property~(ii) demands that $\mathbf{c}_t$ can destabilize one basin and redirect the trajectory to another. As $\mathbf{c}_t$ varies, the landscape $U_\theta(\cdot; \mathbf{c}_t)$ deforms continuously. By elementary bifurcation theory~\cite{strogatz2015nonlinear}, the generic (codimension-one) mechanism by which a smooth potential loses a local minimum is the \emph{saddle-node bifurcation}: a minimum and a neighboring saddle coalesce and annihilate as $\mathbf{c}_t$ crosses a critical value $\mathbf{c}^*$. Near the annihilation point, the one-dimensional normal form is
\begin{equation}
\dot{z} = \mu + z^2,
\label{eq:normal_form}
\end{equation}
where $z \in \mathbb{R}$ is a local coordinate along the bifurcating direction, and $\mu \in \mathbb{R}$ is a scalar bifurcation parameter obtained by projecting $\mathbf{c}_t - \mathbf{c}^*$ onto the critical unstable direction. For $\mu < 0$ two equilibria exist (a stable node and a saddle), at $\mu = 0$ they merge, and for $\mu > 0$ both vanish. In the post-bifurcation regime ($\mu > 0$), the vector field is everywhere nonzero yet remains small near the former equilibrium: $|\dot{z}|$ attains its minimum value $\mu$ at $z = 0$, and the slow-passage bottleneck spans a region of width $\mathcal{O}(\sqrt{\mu})$. This is the \emph{ghost region}, where the trajectory dwells for a time $T_{\mathrm{esc}} \sim \pi/\sqrt{\mu}$ before escaping to a surviving basin (a classical bottleneck-passage result obtained by integrating $\int dz/(\mu + z^2)$~\cite{strogatz2015nonlinear}). The ghost-attractor mechanism thus emerges as a \emph{generic consequence} of smoothly modulating a multi-basin potential; it need not be engineered explicitly. This justifies how Eq.~\eqref{eq:derived_slow} achieves switching: $\mathbf{c}_t$ acts as a bifurcation parameter that reshapes $U_\theta$ to navigate between basins.

\emph{Step 3: Asymmetric transitions require drift.}
Pure gradient flow is quasi-symmetric: the transition landscape from basin $\mathcal{B}_i$ to $\mathcal{B}_j$ is determined entirely by the potential topology, leaving no room for directional preferences. Adding a non-conservative drift $\mathbf{b}_\theta: \mathbb{R}^{d_z} \times \mathbb{R}^{d_c} \to \mathbb{R}^{d_z}$ breaks this constraint, yielding the second term of Eq.~\eqref{eq:derived_slow}.

\subsection{Fast-Slow Decomposition}

The latent state is decomposed into slow and fast components, $\mathbf{x}_t = (\mathbf{z}_t, \mathbf{y}_t) \in \mathbb{R}^{d_z} \times \mathbb{R}^{d_y}$. The slow variable $\mathbf{z}$ inherits the potential-drift dynamics of Eq.~\eqref{eq:derived_slow}, while a fast variable $\mathbf{y}$ tracks a manifold function of $\mathbf{z}$. Specializing the full-state vector field $\mathbf{F}_\theta$ in Eq.~\eqref{eq:full_dynamics} to this two-timescale structure yields:
\begin{align}
\dot{\mathbf{z}} &= \mathbf{f}_{\theta}(\mathbf{z},\mathbf{y},\mathbf{c}_t), \label{eq:slow}\\
\varepsilon\,\dot{\mathbf{y}} &= \mathbf{g}_{\theta}(\mathbf{z},\mathbf{y},\mathbf{c}_t), \quad 0 < \varepsilon \ll 1. \label{eq:fast}
\end{align}
For a frozen context $\mathbf{c}_t$, suppose the fast equation admits a graph of equilibria
\begin{equation}
\mathcal{M}(\mathbf{c}_t) = \left\{(\mathbf{z}, \mathbf{y}) : \mathbf{y} = \mathbf{h}_\theta(\mathbf{z}; \mathbf{c}_t),\; \mathbf{g}_\theta(\mathbf{z}, \mathbf{h}_\theta, \mathbf{c}_t) = \mathbf{0}\right\}.
\label{eq:slow_manifold}
\end{equation}
Under normal hyperbolicity ($\lambda_{\max}(\partial \mathbf{g}_\theta/\partial \mathbf{y}) \le -\gamma$ for some $\gamma > 0$), Fenichel's theorem~\cite{fenichel1979geometric} guarantees that for sufficiently small $\varepsilon$, $\mathcal{M}(\mathbf{c}_t)$ persists as a locally invariant slow manifold $\mathcal{M}_\varepsilon$ that is an $\mathcal{O}(\varepsilon)$ perturbation of $\mathcal{M}(\mathbf{c}_t)$, onto which trajectories collapse at rate $1/\varepsilon$. On this manifold the slow dynamics reduce to $\dot{\mathbf{z}} = \mathbf{f}_\theta(\mathbf{z}, \mathbf{h}_\theta(\mathbf{z}), \mathbf{c}_t)$, which is identified with the potential-drift form of Eq.~\eqref{eq:derived_slow}: the dependence on $\mathbf{y}$ is absorbed via the manifold projection. Since $\varepsilon \ll 1$, the fast variable relaxes almost instantaneously, so the system effectively evolves along the slow manifold. This justifies a decoupled architecture: integrate only $\mathbf{z}$ (Section~\ref{sec:method}), then recover $\mathbf{y}$ via the manifold projection at readout time.

\subsection{Ghost-Induced Switching}

The switching properties of the derived slow dynamics (Eq.~\eqref{eq:derived_slow}), which evolve on the slow manifold $\mathcal{M}_\varepsilon$, are analyzed next. A ghost region $\mathcal{G} \subset \mathcal{M}_\varepsilon$ is a local subset where the potential gradient is small, $\|\nabla_{\mathbf{z}} U_{\theta}(\mathbf{z};\mathbf{c}_t)\| \le \delta$ with $\delta \ll 1$, but no true fixed point exists. The drift $\mathbf{b}_\theta$ is a separate contribution to the velocity (Lemma~\ref{lem:dwell} bounds dwell time using both $\delta$ and $\|\mathbf{b}_\theta\|$). The ``slow evolution'' here is a local bottleneck on top of the already-slow $\mathbf{z}$-timescale of the fast-slow decomposition: within $\mathcal{G}$, even the slow dynamics are further suppressed.

\begin{lemma}[Ghost attractor dwell time]
\label{lem:dwell}
Consider a ghost region $\mathcal{G}$ where $\|\nabla_\mathbf{z} U_\theta\| \le \delta$ and $\|\mathbf{b}_\theta\| \le v_0$. By the triangle inequality applied to the vector sum $\dot{\mathbf{z}} = -\nabla_\mathbf{z} U_\theta + \mathbf{b}_\theta$, the speed satisfies $\|\dot{\mathbf{z}}\| \le \delta + v_0$ throughout $\mathcal{G}$. Hence, for characteristic length $L$ along the trajectory, the dwell time obeys $T_{\mathrm{esc}} \ge L/(\delta + v_0)$. For a saddle-node ghost with negligible drift, direct integration of the normal form (Eq.~\eqref{eq:normal_form}) tightens this to $T_{\mathrm{esc}} \sim \pi/\sqrt{\mu}$ in the post-bifurcation regime $\mu > 0$. (Proof in Appendix~\ref{app:proofs}.)
\end{lemma}

This scaling has a practical consequence: by adjusting $\mathbf{c}_t$, the network tunes transition speed over orders of magnitude without weight changes.

\begin{figure}[!t]
\centering
\includegraphics[width=\columnwidth]{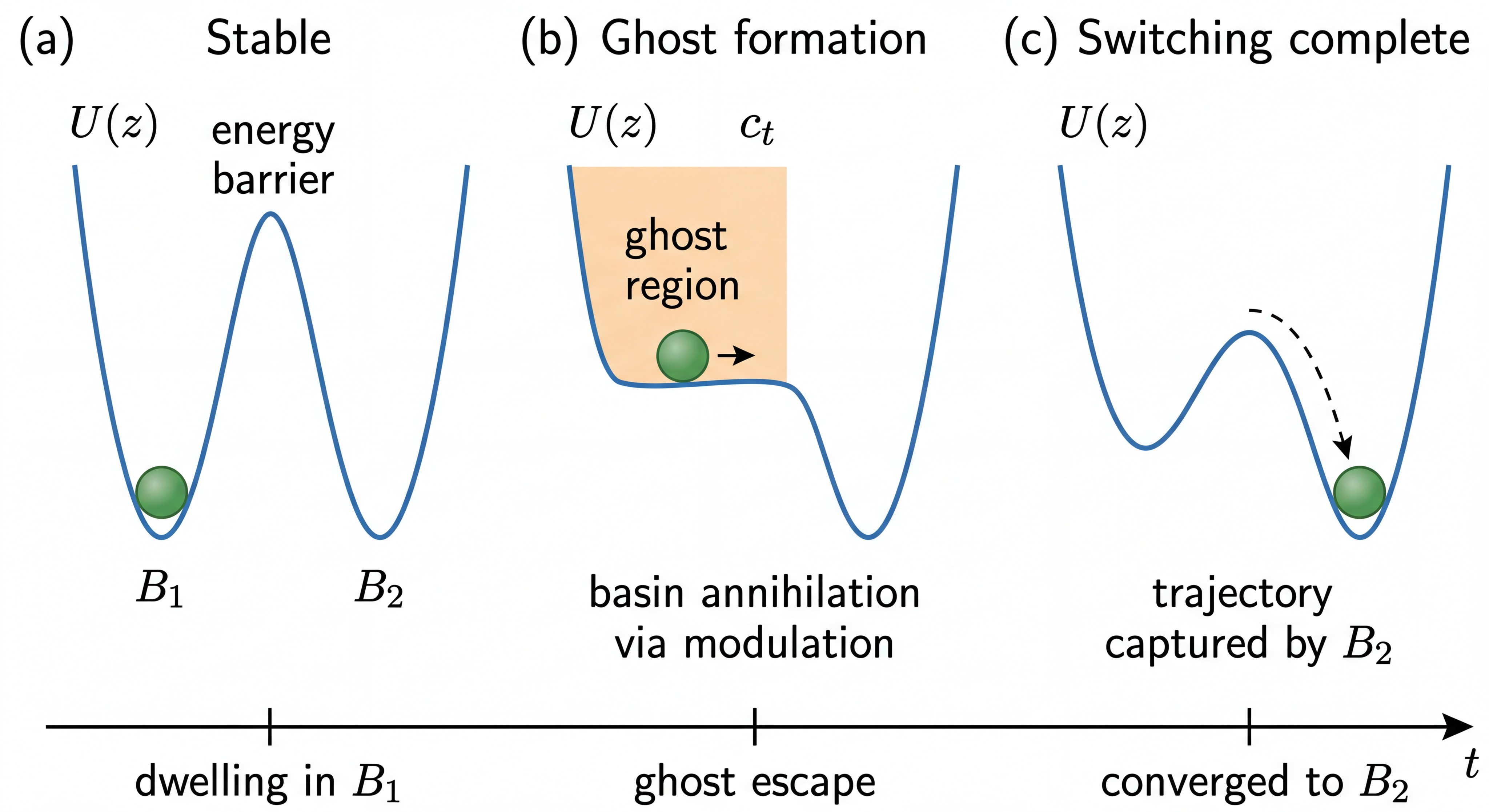}
\caption{Ghost attractor switching mechanism. (a)~Stable: two basins $\mathcal{B}_1$, $\mathcal{B}_2$ separated by an energy barrier. (b)~Modulation $\mathbf{c}_t$ annihilates $\mathcal{B}_1$ into a ghost region (orange) where the trajectory dwells with slow drift. (c)~The trajectory escapes and converges to $\mathcal{B}_2$.}
\label{fig:ghost_mechanism}
\end{figure}

\begin{proposition}[Ghost-mediated switching]
\label{prop:switching}
Let $\mathcal{B}_i$ be a stable basin with equilibrium $\mathbf{z}_s$ and adjacent saddle $\mathbf{z}_u$. After modulation $\mathbf{c}_t \to \mathbf{c}'_t$ (held fixed), suppose: (1) a saddle-node annihilates $\mathbf{z}_s$ and $\mathbf{z}_u$; (2) the resulting ghost channel satisfies $\|\nabla_\mathbf{z} U_\theta\| \le \delta$; (3) the saddle's unstable manifold reaches $\mathcal{B}_j$; (4) drift is dominated by $\|\nabla_\mathbf{z} U_\theta\|$ outside the ghost. Then trajectories starting in a neighborhood of $\mathbf{z}_s$ traverse the ghost and reach $\mathcal{B}_j$ within $T_{\mathrm{esc}}$ (Lemma~\ref{lem:dwell}); trajectories elsewhere may follow other escape routes.
\end{proposition}

Proposition~\ref{prop:switching} establishes that behavioral transitions are mediated by topological changes in the potential landscape, not by accumulating new observations. The modulatory signal $\mathbf{c}_t$ acts as a bifurcation parameter: when it crosses a critical threshold, a basin-saddle pair annihilates and converts a stable attractor into a ghost region (Fig.~\ref{fig:ghost_mechanism}). The trajectory then escapes along the unstable manifold into a neighboring basin. This differs from in-context learning, where transitions require enough new tokens to overwrite prior context.

\begin{corollary}[Combinatorial expressivity]
\label{cor:combinatorial}
Let $U_\theta$ admit $K$ distinct stable basins $\{\mathcal{B}_1, \dots, \mathcal{B}_K\}$. If for each ordered pair $(\mathcal{B}_i, \mathcal{B}_j)$ there exists a modulation path satisfying Proposition~\ref{prop:switching}, then distinguishable behavioral trajectories number $\Omega(K!)$ for trajectories visiting all basins, and $\Omega(2^K)$ for trajectories visiting any subset. (Proof in Appendix~\ref{app:proofs}.)
\end{corollary}

\begin{lemma}[Bounded transition]
\label{lem:bounded}
Assume $U_\theta$ is smooth in $\mathbf{c}_t$, so that the set of bifurcation values is closed and isolated. Let $\mathbf{c}_t$ lie inside the active basin $\mathcal{B}_i$, and for each neighboring basin $\mathcal{B}_j$, let $d_{ij}$ denote the distance in $\mathbf{c}_t$-space from $\mathbf{c}_t$ to the saddle-node bifurcation that annihilates $\mathcal{B}_i$ toward $\mathcal{B}_j$. Let $d_{\min}^{(2)}$ be the second-smallest of $\{d_{ij}\}_j$. If $\|\Delta \mathbf{c}_t\| < d_{\min}^{(2)}$, then the modulated $\mathbf{c}_t + \Delta \mathbf{c}_t$ can cross at most the nearest bifurcation; hence at most one basin can undergo bifurcation per modulation event, guaranteeing stability against cascading transitions.
\end{lemma}

\subsection{Lyapunov Stability and the Stability-Switching Trade-off}
\label{sec:lyapunov}

The potential $V(\mathbf{z}) = U_\theta(\mathbf{z}; \mathbf{c}_t)$ (evaluated at a fixed modulation value $\mathbf{c}_t$) serves as a local Lyapunov function near any basin minimum $\mathbf{z}^*$. The time derivative along trajectories of Eq.~\eqref{eq:derived_slow} is $\dot{V} = -\|\nabla_\mathbf{z} U_\theta\|^2 + \nabla_\mathbf{z} U_\theta \cdot \mathbf{b}_\theta$, which is strictly negative throughout a punctured neighborhood of $\mathbf{z}^*$ whenever $\|\mathbf{b}_\theta\| < \|\nabla_\mathbf{z} U_\theta\|$ for $\mathbf{z} \neq \mathbf{z}^*$ (with $\dot{V} = 0$ only at the equilibrium itself). This condition governs the stability-switching trade-off: deep basins (large $\|\nabla U\|$) are highly stable but require stronger modulation to escape; shallow basins switch readily but are more susceptible to perturbations. When modulation flattens the barrier ($\|\nabla U\| \to 0$), the drift overcomes the weakened gradient and initiates escape. The Lyapunov condition $\nabla U_\theta(\mathbf{z}^*) = 0$ at converged states is automatically satisfied by the gradient-flow architecture itself: trajectories of $\dot{\mathbf{z}} = -\nabla U + \mathbf{b}$ converge to local extrema of $U$ when $\mathbf{b}$ is small relative to $\nabla U$, with no explicit loss required (Section~\ref{sec:method}; ablation in Section~\ref{subsec:exp_offline}).

\section{Architecture and Training}
\label{sec:method}

Following the theoretical derivation above, the proposed architecture is instantiated as a neural decoder module that drops into existing sequential generation pipelines. Fig.~\ref{fig:method} gives an overview.

\begin{figure}[!t]
\centering
\includegraphics[width=\columnwidth]{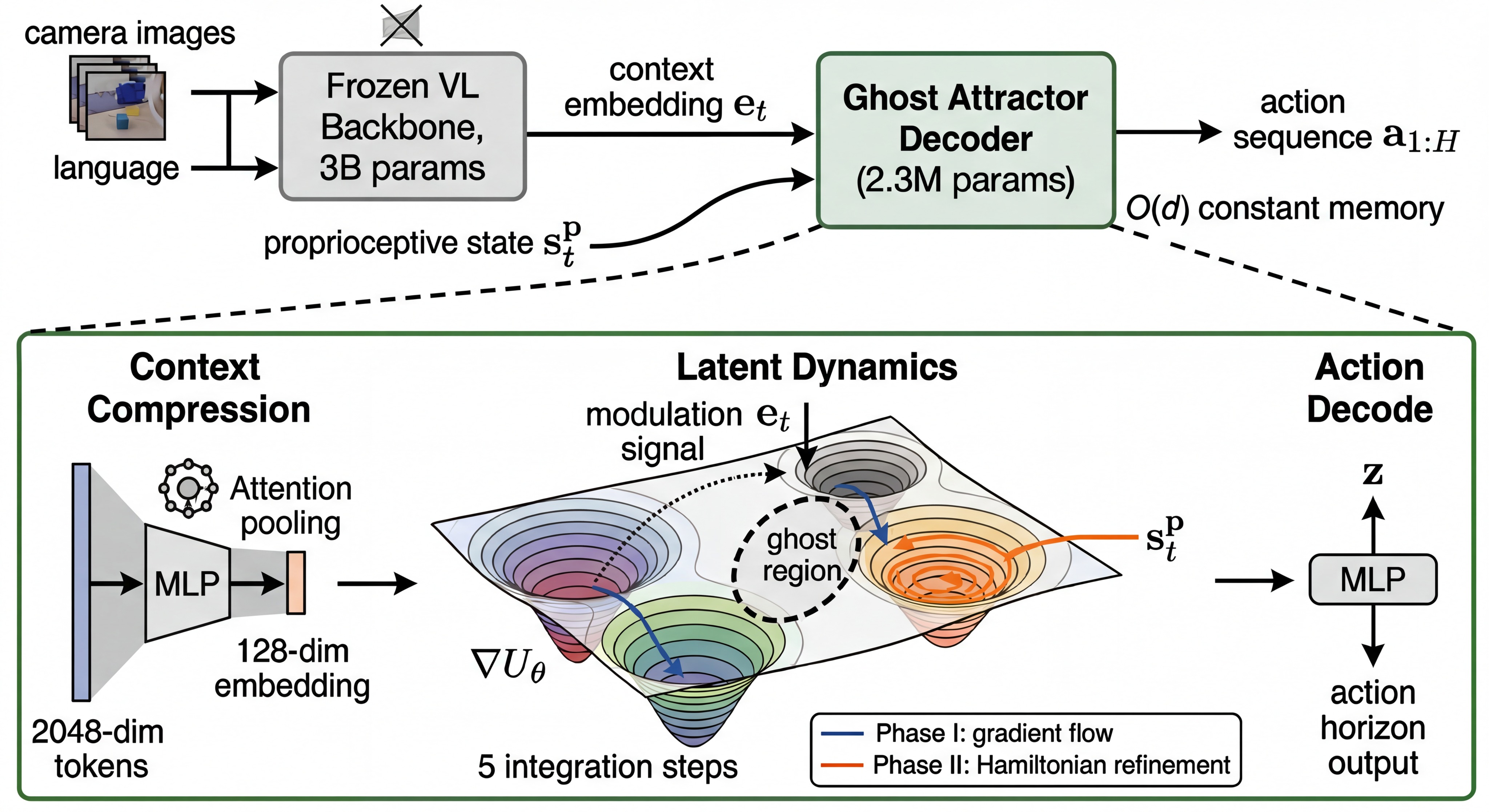}
\caption{Architecture overview. \textbf{Top:} Frozen backbone produces context $\mathbf{e}_t$; with proprioception $\mathbf{s}_t^p$, the Ghost decoder (2.3M) generates actions in $\mathcal{O}(d_z)$ memory. \textbf{Bottom:} Latent dynamics evolve on $U_\theta$ via Phase~I (blue, basin convergence) and Phase~II (orange, proprioceptive refinement). Ghost regions mediate basin transitions.}
\label{fig:method}
\end{figure}

\subsection{Integration with Existing Architectures}

Large-scale sequential generation models share a pipeline: pretrained encoders convert inputs into token sequences, and a temporal-context decoder converts these tokens into output sequences. The proposed dynamical decoder replaces only the temporal-context block. The multimodal tokens (visual $\mathbf{T}^v_t$, proprioceptive $\mathbf{T}^p_t$, language $\mathbf{T}^l$, and an optional task-conditioning vector $\boldsymbol{\tau}_t$) are fused into a fixed-dimensional context embedding $\mathbf{e}_t \in \mathbb{R}^{d_e}$:
\begin{equation}
\mathbf{e}_t = \text{MLP}_{\mathrm{fuse}}([\text{pool}(\mathbf{T}^v_t); \mathbf{T}^p_t; \text{pool}(\mathbf{T}^l); \boldsymbol{\tau}_t]),
\end{equation}
replacing quadratic self-attention with a constant-time operation. The output $\mathbf{e}_t$ instantiates the theoretical modulatory signal $\mathbf{c}_t$ (with $d_e$ playing the role of $d_c$) and reshapes the potential landscape at every timestep.

\subsection{Neural Instantiation of the Dynamical System}

The theoretical slow dynamics (Eq.~\eqref{eq:derived_slow}) are instantiated with concrete neural network components: $U_\theta: \mathbb{R}^{d_z} \times \mathbb{R}^{d_e} \to \mathbb{R}$ is a scalar-valued MLP (the potential), $\mathbf{b}_\theta: \mathbb{R}^{d_z} \times \mathbb{R}^{d_e} \to \mathbb{R}^{d_z}$ is a vector-valued MLP (the drift), and $\mathbf{c}_t$ is instantiated as $\mathbf{e}_t$ (Section~\ref{sec:method}.A). An initial state encoder $\mathbf{E}_\theta: \mathbb{R}^{d_s} \to \mathbb{R}^{d_z}$ maps the proprioceptive state $\mathbf{s}_t^p$ to the starting latent $\mathbf{z}^{(0)} = \mathbf{E}_\theta(\mathbf{s}_t^p)$. The fast variable $\mathbf{y}$ from the theoretical fast-slow decomposition is not explicitly integrated in the architecture; the singular-perturbation property ($\varepsilon \ll 1$) is exploited to reduce dynamics to the slow manifold, and the per-timestep output is recovered via a learned readout (Section~\ref{sec:action_decoding}). Unlike neural ODEs~\cite{chen2018neuralode} that learn unconstrained vector fields, the potential structure constrains the dynamics to the form derived in Section~\ref{sec:theory}, where basin existence and ghost-attractor switching are built into the architecture.

\subsection{Hierarchical Phase-Space Integration}

Both phases operate within the potential-drift framework derived in Section~\ref{sec:theory}, but at different dynamical orders suited to different gradient regimes. Because $\varepsilon \ll 1$, the fast variable $\mathbf{y}$ relaxes to $\mathbf{h}_\theta(\mathbf{z}; \mathbf{e}_t)$ on a timescale much shorter than the integration step size, so both phases evolve only the slow variable $\mathbf{z}$. The fast-slow separation justifies a decoupled design: integration determines the basin (via $\mathbf{z}$), and a separate learned mapping produces the output (Section~\ref{sec:action_decoding}).

Phase I (Basin convergence, first-order dynamics): The first $n_1$ steps discretize the slow dynamics (Eq.~\eqref{eq:derived_slow}, with $\mathbf{c}_t = \mathbf{e}_t$) via explicit Euler, starting from $\mathbf{z}^{(0)} = \mathbf{E}_\theta(\mathbf{s}_t^p)$:
\begin{equation}
\mathbf{z}^{(k+1)} = \mathbf{z}^{(k)} + \Delta t_1 \bigl(-\nabla_\mathbf{z} U_\theta(\mathbf{z}^{(k)}; \mathbf{e}_t) + \mathbf{b}_\theta(\mathbf{z}^{(k)}; \mathbf{e}_t)\bigr),
\label{eq:phase1}
\end{equation}
where $\Delta t_1$ is the Phase~I Euler step size (instantiating the generic $\Delta t$ in Eq.~\eqref{eq:full_dynamics}). First-order (overdamped) dynamics are efficient when the potential gradient is strong, driving $\mathbf{z}$ rapidly toward the appropriate basin.

Phase II (Proprioceptive refinement, second-order dynamics): After Phase~I, $\mathbf{z}$ is near a basin floor where $\nabla_\mathbf{z} U_\theta \approx \mathbf{0}$. First-order dynamics stall in this regime because the driving force vanishes. The integrator is therefore promoted to second-order dissipative Hamiltonian dynamics $\dot{\mathbf{z}} = \mathbf{p}/m, \,\dot{\mathbf{p}} = -\nabla_\mathbf{z} V_\theta - \eta\,\mathbf{p} + \boldsymbol{\phi}_\theta$, introducing momentum $\mathbf{p} \in \mathbb{R}^{d_z}$ that allows continued refinement through inertia. The modulatory signal is enriched with proprioceptive feedback, $\mathbf{c}_t \to \boldsymbol{\xi}_t := (\mathbf{e}_t, \mathbf{s}_t^p)$ where $\mathbf{s}_t^p \in \mathbb{R}^{d_s}$, reshaping the local basin geometry for precision. In discrete form, the remaining $n_2$ steps evolve $(\mathbf{z}, \mathbf{p})$ as:
\begin{align}
\mathbf{p}^{(k+1)} &= \mathbf{p}^{(k)} + \Delta t_2 \bigl(-\nabla_\mathbf{z} V_\theta - \eta\, \mathbf{p}^{(k)} + \boldsymbol{\phi}_\theta\bigr), \label{eq:phase2_p} \\
\mathbf{z}^{(k+1)} &= \mathbf{z}^{(k)} + \Delta t_2 \, \mathbf{p}^{(k+1)}/m,
\label{eq:phase2_z}
\end{align}
where both $V_\theta$ and $\boldsymbol{\phi}_\theta$ are evaluated at $(\mathbf{z}^{(k)}; \boldsymbol{\xi}_t)$. Here $V_\theta: \mathbb{R}^{d_z} \times \mathbb{R}^{d_e} \times \mathbb{R}^{d_s} \to \mathbb{R}$ is a refinement potential (a separate MLP from $U_\theta$, shaping local geometry around each basin rather than global basin topology), $\boldsymbol{\phi}_\theta: \mathbb{R}^{d_z} \times \mathbb{R}^{d_e} \times \mathbb{R}^{d_s} \to \mathbb{R}^{d_z}$ is a learned external-force MLP playing the role of drift in the second-order dynamics (analogous to $\mathbf{b}_\theta$ in Phase~I; distinct from the slow-dynamics function $\mathbf{f}_\theta$ in Eq.~\eqref{eq:slow}), $\eta > 0$ is the damping coefficient ensuring convergence, and $m > 0$ is mass. The proprioceptive state $\mathbf{s}_t^p$ is injected directly into $V_\theta$ alongside $\mathbf{e}_t$: while the compressed proprioceptive tokens inside $\mathbf{e}_t$ provide coarse state context for backbone fusion, the raw $\mathbf{s}_t^p$ offers high-fidelity feedback specifically for the precision-sensitive Phase~II refinement. The Phase~II step size $\Delta t_2$ similarly instantiates the generic $\Delta t$. Phase~II uses iteration index $k = 0, \dots, n_2 - 1$ continuing from $k = n_1$ in Phase~I; that is, the Phase~II initial state is the Phase~I terminal state $\mathbf{z}^{(n_1)}$, with momentum $\mathbf{p}^{(n_1)} = \mathbf{0}$ (starting from rest). This second-order extension remains within the potential-drift framework: the same potential structure drives the dynamics, promoted from first-order $\dot{\mathbf{z}} \propto -\nabla V$ to second-order Hamiltonian flow with momentum to handle the vanishing-gradient regime near basin minima~\cite{greydanus2019hamiltonian}. The dissipation term $\eta\mathbf{p}$ guarantees that the refinement converges rather than oscillating indefinitely~\cite{shadmehr2005computational}.

\subsection{Training Objective}

The training objective combines standard behavioral cloning with a single auxiliary contrastive term:
\begin{equation}
\mathcal{L}_{\mathrm{total}} = \underbrace{\|\hat{\mathbf{a}} - \mathbf{a}^*\|^2}_{\mathcal{L}_{\mathrm{BC}}} + \lambda_c \mathcal{L}_{\mathrm{con}},
\label{eq:loss}
\end{equation}
where $\mathbf{a}^*$ is the ground-truth demonstration action and $\mathcal{L}_{\mathrm{con}}$ is an NT-Xent contrastive loss~\cite{chen2020simclr} on the post-integration latent $\mathbf{z}^*$ that pulls same-task states together and pushes different-task states apart. The contrastive term encourages the potential landscape to develop one basin per task. The basin/ghost topology then emerges from two sources: (i)~the architectural inductive bias of Eq.~\eqref{eq:derived_slow}, in which gradient flow on a scalar potential automatically converges to local minima ($\dot{U}_\theta \le 0$), and (ii)~the task-distinct supervision provided by $\mathcal{L}_{\mathrm{BC}}$ and $\mathcal{L}_{\mathrm{con}}$. No explicit topology-shaping loss (e.g., $\|\nabla U(\mathbf{z}^*)\|^2$ or ghost-channel direction supervision) is required. This matches the generic-consequence interpretation in Section~\ref{sec:theory}.A and is verified empirically by the ablation in Section~\ref{subsec:exp_offline}. Training uses Adam optimization with $\lambda_c = 0.2$, learning rate $3 \times 10^{-4}$, batch size 32, and 300 epochs.

\subsection{Action Readout}
\label{sec:action_decoding}

The fast-slow theory motivates a decoupled decoding design: Phase~I and~II determine the converged slow state $\mathbf{z}^{(n)} \in \mathbb{R}^{d_z}$, and a separate learned mapping produces the output. Concretely, the theoretical action map $\boldsymbol{\pi}_\theta$ in Eq.~\eqref{eq:full_dynamics} is realized as the composition $\boldsymbol{\pi}_\theta = \text{MLP}_{\mathrm{decode}} \circ \text{MLP}_{\mathrm{horizon}}$, applied per horizon position. Inspired by the manifold projection $\mathbf{h}_\theta$ that maps the slow variable to the fast variable in the theoretical formulation, a horizon network is used to expand the converged latent into $H$ per-timestep intermediate representations:
\begin{equation}
\mathbf{u}_{1:H} = \text{MLP}_{\mathrm{horizon}}(\mathbf{z}^{(n)}) \in \mathbb{R}^{H \times d_z}.
\end{equation}
Each $\mathbf{u}_h \in \mathbb{R}^{d_z}$ is decoded per timestep:
\begin{equation}
\hat{\mathbf{a}}_h = \text{MLP}_{\mathrm{decode}}(\mathbf{u}_h) \in \mathbb{R}^{d_a}, \quad h = 1, \dots, H,
\end{equation}
producing $H$ future action vectors $\hat{\mathbf{a}}_h$ of action dimension $d_a$ simultaneously. Note that $\mathbf{u}_h$ differs from the theoretical fast variable $\mathbf{y}$: while $\mathbf{y}$ is the instantaneous manifold state at time $t$, $\mathbf{u}_h$ indexes a future horizon position. The horizon network can be viewed as a learned temporal unfolding of the manifold projection tailored to action chunking. In the experiments, $d_z = 64$, $H = 17$ (matching the GR00T N1.5 action chunk length), and $d_a = 32$ (17 real action dimensions padded to 32).

The complete architecture consists of: (i) a VL context compressor (two-layer MLP, $2048 \to 512 \to 128$) with attention pooling over VL tokens; (ii) a state encoder $\mathbf{E}_\theta$ (two-layer MLP, $64 \to 512 \to 64$) mapping proprioceptive state $\mathbf{s}_t^p$ to the initial latent $\mathbf{z}^{(0)}$; (iii) a potential network $U_\theta$ (three-layer MLP, $192 \to 512 \to 512 \to 1$, where 192 = 128 context + 64 latent); (iv) a drift network $\mathbf{b}_\theta$ (two-layer MLP, $192 \to 512 \to 64$); (v) a horizon network (two-layer MLP, $64 \to 512 \to H \times 64$); and (vi) an action decoder $\text{MLP}_{\mathrm{decode}}$ (two-layer MLP, $64 \to 512 \to 32$). Total parameter count: 2,315,746. The architecture achieves $\mathcal{O}(d_z)$ memory and $\mathcal{O}(n \cdot d_z \cdot d_h)$ compute per control step (where $n$ is the number of integration steps and $d_h$ the hidden width), independent of horizon length, in contrast to the $\mathcal{O}(t)$ scaling of Transformer decoders. This property is quantified empirically in Section~\ref{sec:results}.

Table~\ref{tab:theory_method_map} summarizes the correspondence between theoretical constructs of Section~\ref{sec:theory} and the architectural components introduced above.

\begin{table}[!t]
\centering
\caption{Mapping from theoretical constructs to architectural components.}
\label{tab:theory_method_map}
\setlength{\tabcolsep}{4pt}
\begin{tabular}{@{}lll@{}}
\toprule
Theory & Equation / Result & Architecture \\
\midrule
Modulatory signal $\mathbf{c}_t$ & Eq.~\eqref{eq:full_dynamics} & Cross-modal fusion $\mathbf{e}_t$ \\
Field $\mathbf{F}_\theta$ & Eq.~\eqref{eq:full_dynamics} & Potential + drift MLPs \\
Slow variable $\mathbf{z}_t$ & Eq.~\eqref{eq:slow} & Latent state ($d_z{=}64$) \\
Fast variable $\mathbf{y}_t$ & Eq.~\eqref{eq:fast} & Horizon unfolding (\S\ref{sec:action_decoding}) \\
Potential $U_\theta$ & Eq.~\eqref{eq:derived_slow} & Scalar MLP \\
Drift $\mathbf{b}_\theta$ & Eq.~\eqref{eq:derived_slow} & Vector MLP \\
Ghost regions & Lemma~\ref{lem:dwell} & Emergent under $\mathbf{e}_t$ modulation \\
Refinement potential $V_\theta$ & Eqs.~\eqref{eq:phase2_p}--\eqref{eq:phase2_z} & Phase~II second-order MLP \\
\bottomrule
\end{tabular}
\end{table}

\section{Experiments}
\label{sec:results}

\subsection{Application Domain and Datasets}

Ghost is evaluated on real-time robot action decoding, where a pretrained vision-language backbone produces context embeddings and a decoder must generate multi-step motor commands at high frequency. Among sequential decoding domains, robotics is chosen as the focus because it stresses all three properties the architecture targets: closed-loop control with environment feedback, naturally multi-phase modal structure (approach/grasp/lift/place sub-phases), and hard real-time and memory constraints from physical deployment. Adjacent domains exercise only subsets of these conditions: language generation lacks the closed-loop feedback, and offline planning lacks the real-time budget. This makes robotics a particularly stringent test bed. The output space is mixed continuous-discrete, multi-modal, and requires instant switching between qualitatively different behavioral modes.

Three evaluation settings of increasing complexity are used. (i)~Synthetic environments (2D potential landscapes and goal-conditioned switching tasks) validate theoretical predictions independently of real-world confounds (Section~\ref{subsec:exp_theory}). (ii)~An in-house dual-arm humanoid manipulation dataset covering 10 tasks across five scene types: cup arrangement, conveyor-belt part sorting, counter and conference-room tidying, toy / tool / book organization, store cleaning, and toast preparation. This is the same 2{,}196-episode dataset used to fine-tune GR00T N1.5. Section~\ref{subsec:exp_offline} draws two splits from it: an episode-level split (1{,}800 train / 396 unseen episodes) for the Diffusion-Transformer drop-in replacement, and a per-task sample-level split (300--600 train / 80--150 test samples per task; 5{,}600 train + 1{,}430 test = 7{,}030 samples in total) for the multi-task and ablation experiments. All offline experiments use cached features from the GR00T N1.5 foundation model~\cite{bjorck2025groot}, a 3B-parameter vision-language-action backbone; this isolates the decoder as the experimental variable. (iii)~LIBERO-10~\cite{liu2024libero}, a standard closed-loop benchmark with 10 tabletop manipulation tasks, evaluates Ghost as an action decoder in realistic robotic deployment, with success rate as the primary metric (Section~\ref{subsec:exp_libero}).

The switching mechanism, central to the paper's claims, is probed at four progressively realistic scales: synthetic potential dynamics (Section~\ref{subsec:exp_theory}), goal-conditioned meta-RL (Section~\ref{subsec:exp_theory}), multi-task pairwise transitions on real-world robot data (Section~\ref{subsec:exp_offline}), and closed-loop language-instructed switching (Section~\ref{subsec:exp_libero}).

\subsection{Theoretical Validation}
\label{subsec:exp_theory}

Ghost's theoretical properties are first validated on diagnostic and synthetic environments before evaluation on real robotic systems.

\paragraph{Diagnostic Visualization of Attractor Dynamics}

To validate theoretical predictions, a diagnostic environment is constructed with $K = 3$ behavioral modes in a 2D latent space, allowing direct visualization. Three modes correspond to distinct velocity targets. The potential is parameterized as modulated Gaussian wells with quadratic confinement:
\begin{equation}
U(\mathbf{z}; \mathbf{c}) = \alpha \|\mathbf{z}\|^2 - \sum_{k=1}^{3} w_k(\mathbf{c}) \exp\!\left(-\frac{\|\mathbf{z} - \boldsymbol{\mu}_k\|^2}{2\sigma^2}\right)
\label{eq:analytic_potential}
\end{equation}

Fig.~\ref{fig:diagnostic}(a) shows the potential with all basins active: three wells separated by ridges. Fig.~\ref{fig:diagnostic}(b) shows the landscape after collapsing basin $\mathcal{B}_1$ into a ghost region: gradient arrows near-zero, confirming $\|\nabla_\mathbf{z} U\| \le \delta$. A trajectory initialized near the collapsed basin dwells approximately 3.5~s before escaping to basin $\mathcal{B}_2$ (Fig.~\ref{fig:diagnostic}(c)), consistent with the dwell-time scaling of Lemma~\ref{lem:dwell} and the ghost-mediated switching mechanism of Proposition~\ref{prop:switching}. Action signals exhibit the characteristic ghost-attractor profile: stable output, delay followed by rapid switch, then re-stabilization without oscillation.

To check whether training recovers multi-basin topology from data alone, a neural potential network (3-layer MLP, 128-128-64 hidden units, SiLU activations) is trained on behavior-cloning data generated by the analytic system. The training procedure uses only the standard behavioral cloning loss $\mathcal{L}_{\mathrm{BC}} = \|\hat{\mathbf{a}} - \mathbf{a}^*\|^2$ augmented with the slow-manifold and ghost-structure regularizers described in Section~\ref{sec:method}; no supervision is placed on the shape of the potential landscape itself.

The learned network reaches a final loss of 0.003 and produces abrupt switching within two control steps once the context is modulated. This experiment confirms a load-bearing assumption of the theoretical framework: the topology-shaping training objectives (Section~\ref{sec:method}) are sufficient to induce multi-basin potential structure from behavioral data alone, without specifying attractor locations or transition paths.

\begin{figure*}[!t]
\centering
\begin{tabular}{@{}cccc@{}}
\includegraphics[width=0.232\textwidth]{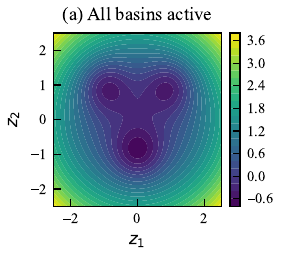} &
\includegraphics[width=0.232\textwidth]{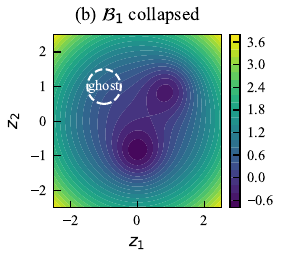} &
\includegraphics[width=0.232\textwidth]{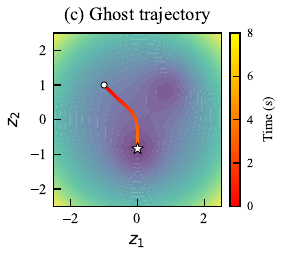} &
\includegraphics[width=0.232\textwidth]{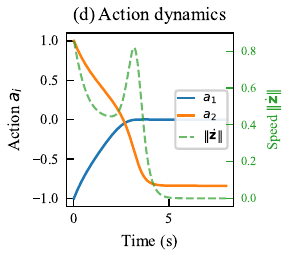} \\
\end{tabular}
\caption{Diagnostic visualization in 2D latent space, $K = 3$ modes. (a) Potential with all basins active. (b) Modulated landscape: $\mathcal{B}_1$ collapsed into ghost region (dashed). (c) Trajectory escaping the ghost remnant and converging to $\mathcal{B}_2$; color encodes time. (d) Action signals (left) and latent speed $\|\dot{\mathbf{z}}\|$ (right): dip at $t\!\approx\!1.5$s + peak at $t\!\approx\!3$s is the dwell--escape signature.}
\label{fig:diagnostic}
\end{figure*}

\paragraph{Switching Speed: Ghost vs Recurrent vs Feedforward}
To quantify the switching advantage of potential-based dynamics, convergence speed after an abrupt context switch is compared on a tilted quartic double-well potential:
\begin{equation}
U(z_1; c) = 0.3(z_1^2 - d^2)^2 - 3(2c-1) z_1
\end{equation}
where $d = 1.5$ and $c \in \{0, 1\}$ selects the active basin. At $t = T_\text{switch}$, the context switches from $c = 0$ (Basin~A at $z_1 \approx -1.9$) to $c = 1$ (Basin~B at $z_1 \approx +1.9$), and the number of steps to converge within distance 0.15 of the new basin center is measured.

Ghost ODE integration ($\mathrm{d}\mathbf{z}/\mathrm{d}t = -\nabla U$) converges in 1--3 steps depending on integration depth ($n = 10$: 1 step; $n = 5$: 2 steps; $n = 3$: 3 steps). A GRU-like recurrent update $\mathbf{h}_t = (1 - \alpha)\mathbf{h}_{t-1} + \alpha \boldsymbol{\mu}_B$ requires $>$20 steps even at $\alpha = 0.3$, because convergence is geometric with constant contraction ratio $(1 - \alpha)$.

The mechanism difference is revealed by the contraction ratio $d_t / d_{t-1}$, where $d_t$ is the distance to the target basin after step $t$. Ghost exhibits an \emph{accelerating} regime: the ratio decreases from $\approx 0.85$ at step~1 to a minimum of $\approx 0.25$ at steps~4--5, while GRU maintains a constant ratio of $\approx 0.72$ throughout. The non-monotone shape of Ghost's contraction profile reflects the gradient field structure on a quartic potential: contraction is fastest in the steep region between the saddle and the target basin, slowing once the trajectory enters the basin's quadratic neighborhood (Fig.~\ref{fig:switch_speed}b).

\begin{figure}[!t]
\centering
\includegraphics[width=\columnwidth]{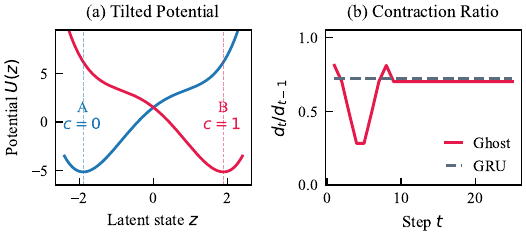}
\caption{Switching speed on a synthetic quartic double-well potential. (a) Tilted landscape under contexts $c\!=\!0$ and $c\!=\!1$; the global minimum moves Basin~A$\to$Basin~B. (b) Contraction ratio $d_t / d_{t-1}$: Ghost shows accelerating convergence near the basin; GRU maintains constant geometric decay ($\approx 0.72$).}
\label{fig:switch_speed}
\end{figure}

\paragraph{Goal-Conditioned Switching Benchmarks}

Evaluation uses a goal-conditioned switching benchmark in which goal locations change mid-episode, with Algorithm Distillation (AD)~\cite{laskin2022incontext}, a Transformer baseline, and RL$^2$~\cite{duan2016rl2} as comparisons.

The Ghost decoder achieves switching latency of $10.8 \pm 3.4$ steps, compared to $66.0 \pm 65.2$ for Transformer baseline, $170.0 \pm 72.7$ for AD, and $219.4 \pm 12.0$ for RL$^2$ (Table~\ref{tab:meta_rl}). This represents 6x improvement over the nearest competitor, confirming that ghost-attractor dynamics enable transitions at control-loop timescales rather than requiring context accumulation.

Steady-state reward of 1.00 exceeds all baselines (Transformer: 0.31, AD: 0.06, RL$^2$: -0.08), and total episode reward of 281.5 exceeds Transformer by 4x. With 17,922 parameters compared to 417,538 for Transformer, Ghost achieves a $23\times$ parameter reduction while attaining the highest reward in Table~\ref{tab:meta_rl}.

On Meta-World ML10, a standard multi-task benchmark with 10 manipulation tasks, all context-conditioned agents change action output within 1--2 steps of task switch since they receive task identifier as input. However, Ghost achieves 0.6--0.8 cosine alignment with new-task expert from the very first step, while Transformer and AD produce negatively aligned actions ($-0.4$ to $0.0$). The ghost-attractor advantage is therefore switching \emph{accuracy} rather than switching \emph{speed}. When an agent changes output quickly but in the wrong direction, the adaptation can be counterproductive. The potential landscape ensures that the first post-switch action already lies near the correct basin, rather than requiring iterative context accumulation to identify the new behavioral mode.

The reward adaptation curves (Fig.~\ref{fig:meta_rl}) make the practical consequences visible: Ghost recovers to full performance within roughly 10 steps, Transformer needs 50--100 steps, and AD needs 150--200+ steps. For a robot operating at 100~Hz, 10 steps corresponds to 100~ms of suboptimal behavior compared with 1--2 seconds for the baselines. In time-sensitive applications this difference is operationally non-trivial.

\begin{table}[!t]
\centering
\caption{Goal-conditioned switching benchmark. Switching latency is steps to reach new goal after context switch. All values averaged over 20 episodes.}
\label{tab:meta_rl}
\resizebox{\columnwidth}{!}{%
\begin{tabular}{lcccc}
\toprule
Method & Switch Latency & Steady Reward & Total Reward & Params \\
\midrule
AD & $170.0 \pm 72.7$ & 0.06 & 23.0 & 421,378 \\
Transformer & $66.0 \pm 65.2$ & 0.31 & 70.5 & 417,538 \\
RL$^2$ & $219.4 \pm 12.0$ & -0.08 & -22.4 & 54,533 \\
\midrule
Ghost (proposed) & $10.8 \pm 3.4$ & 1.00 & 281.5 & 17,922 \\
\bottomrule
\end{tabular}}
\end{table}

\begin{figure}[!t]
\centering
\includegraphics[width=\columnwidth]{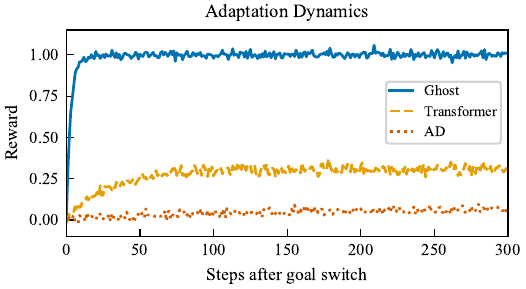}
\caption{Reward adaptation curves following goal switch at step 0. Ghost recovers to $\approx 1.0$ within $\approx 10$ steps, Transformer plateaus near $0.3$, and AD fails to recover ($\approx 0$). Switching latencies and parameter counts are reported in Table~\ref{tab:meta_rl}.}
\label{fig:meta_rl}
\end{figure}

\paragraph{Attractor Emergence in Trained Ghost}
The diagnostic experiment above (Fig.~\ref{fig:diagnostic}) shows that Ghost's potential-drift form admits attractor dynamics in a controlled 2D setting. The same property is now established in a Ghost trained end-to-end on the 10-task data with $\mathcal{L}_{\mathrm{BC}} + \lambda_c \mathcal{L}_{\mathrm{con}}$ (no auxiliary basin-convergence loss; see Section~\ref{subsec:exp_offline}). For each of the 1{,}430 held-out test samples on the 10-task data, the potential gradient norm $\|\nabla_z U(z^{(t)}; c)\|_2$ is traced across the five integration steps. A genuine gradient-flow contraction toward a basin minimum should produce a monotonically decreasing curve approaching the local-minimum condition $\nabla U \to 0$. Fig.~\ref{fig:attractor_emergence} reports results on two trained variants. Under the standard training objective, the gradient norm decays from $1.68$ at $t\!=\!0$ to $0.55$ at $t\!=\!5$ ($67\%$ reduction), matching the gradient-flow contraction predicted by Eq.~\eqref{eq:derived_slow}. The potential gradient is comparable in magnitude to the drift ($\|\nabla U\|/\|b\|=0.62$ at $t\!=\!0$), so the dynamics are jointly shaped by both terms rather than dominated by drift. When training adds a $0.01\|\nabla U(z^*)\|^2$ basin-convergence regularizer, by contrast, the gradient norm collapses to $\approx 0.04$ at every integration step ($\|\nabla U\|/\|b\|=0.015$) with no decay, indicating a flattened potential in which attractor dynamics are no longer empirically observable. The basin term, intended to enforce $\nabla U(z^*)\to 0$ at converged states, instead drives $U$ to a globally flat solution, at which the gradient-flow integrator stops contributing to convergence. Section~\ref{sec:discussion} returns to this empirical sensitivity to the training recipe.

\begin{figure}[!t]
\centering
\includegraphics[width=0.85\columnwidth]{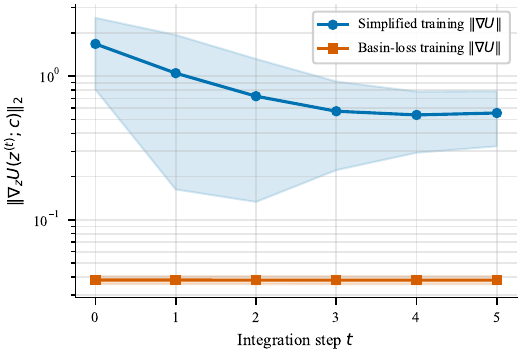}
\caption{Empirical attractor verification on a Ghost trained end-to-end. $\|\nabla_z U(z^{(t)}; c)\|_2$ traced across five integration steps on 1{,}430 held-out samples (mean $\pm$ std). Simplified training ($\mathcal{L}_{\mathrm{BC}}+\lambda_c \mathcal{L}_{\mathrm{con}}$, blue) shows a $67\%$ contraction $1.68\!\to\!0.55$ consistent with gradient-flow convergence; adding a $0.01\|\nabla U(z^*)\|^2$ basin regularizer (orange) flattens the potential to $\|\nabla U\|\approx 0.04$ throughout, suppressing observable attractor dynamics.}
\label{fig:attractor_emergence}
\end{figure}

\subsection{Offline Decoder Comparison on Real-Robot Data}
\label{subsec:exp_offline}

With the theoretical properties validated, Ghost is now evaluated as a practical decoder replacement using cached backbone features from the GR00T N1.5 foundation model on real-world robot demonstration data.

\paragraph{Drop-in Replacement for Billion-Parameter Decoders}

The Ghost decoder is evaluated as a drop-in replacement for the Diffusion Transformer decoder in GR00T N1.5~\cite{bjorck2025groot}. The original model consists of a 3B-parameter vision-language backbone and a 1.07B-parameter DiT decoder. The backbone is frozen and the Ghost decoder is trained on cached features from 1,800 demonstration episodes, using combined 50\% knowledge distillation and 50\% ground-truth supervision. Evaluation is on 200 held-out timesteps from 396 unseen episodes.

The Ghost decoder matches DiT on overall ground-truth MSE (0.116 vs 0.117 across 17 action dimensions; Table~\ref{tab:groot_distill}) while using 462x fewer parameters. The error decomposes asymmetrically across dimension types: DiT remains stronger on the 14 continuous arm joints (0.051 vs 0.105), while Ghost substantially outperforms on the discrete dimensions, including grippers (0.119 vs 0.137) and the lift command (0.263 vs 1.004). The lift dimension is informative: DiT's iterative denoising regresses toward the training-distribution mean, while Ghost's deterministic gradient integration preserves constant signals.

Three observations stand behind Ghost's value here (Table~\ref{tab:groot_distill}). (i)~\emph{Efficiency:} 462$\times$ fewer parameters and 32$\times$ lower inference latency on an NVIDIA RTX 5090 (0.41~ms vs.\ 13~ms for DiT's 4-step Euler denoising). At 100~Hz control, DiT latency consumes the entire time budget; Ghost leaves 9.7~ms for other computation. (ii)~\emph{Structural advantage on discrete dimensions:} Ghost's deterministic gradient integration preserves the discrete gripper and lift signals on which DiT's denoising regresses toward the distribution mean, yielding 1.2$\times$--3.8$\times$ MSE reductions on those commands. (iii)~\emph{No accuracy loss:} the discrete-dimension wins more than offset DiT's slight edge on continuous arm joints, leaving the 17-dimension aggregate MSE essentially tied (0.116 vs.\ 0.117). Sample trajectories (Fig.~\ref{fig:e2e_trajectory}) show qualitatively that the 2.3M-parameter Ghost decoder tracks ground-truth demonstrations more closely than the 1.07B-parameter DiT on each of the four representative dimensions shown.

The contrasting behavior on continuous versus discrete dimensions points to an architectural reason. DiT's iterative denoising regresses toward the training-distribution mean on low-variance dimensions: when the lift command is consistently $-1.0$ during demonstrations, DiT predicts values near $0$ (the distribution center), yielding MSE of 1.004. Ghost's deterministic gradient integration preserves the constant signal because the potential landscape encodes a deep, narrow basin at $-1.0$. This suggests that dynamical-system decoders have a structural advantage for output spaces with mixed continuous-discrete semantics, which are common in real-world robot control (joint positions, gripper commands, mode flags).

\begin{table}[!t]
\centering
\caption{Ghost vs. 1.07B-parameter DiT decoder in GR00T N1.5. Both trained on episodes 0--1799; metrics on 200 held-out timesteps from unseen episodes.}
\label{tab:groot_distill}
\begin{tabular}{lccc}
\toprule
Metric & Ghost & DiT & Comparison \\
\midrule
Decoder Params & 2.3 M & 1068.8 M & 462x fewer \\
GT MSE (total) & 0.116 & 0.117 & on par \\
\quad 14 arm joints & 0.105 & 0.051 & DiT 2.1x \\
\quad 2 grippers & 0.119 & 0.137 & 1.2x lower \\
\quad lift command & 0.263 & 1.004 & 3.8x lower \\
GT Cosine Similarity & 0.872 & 0.887 & on par \\
Inference Latency & 0.41 ms & 13 ms & 32x faster \\
\bottomrule
\end{tabular}
\end{table}

\begin{figure}[!t]
\centering
\includegraphics[width=\columnwidth]{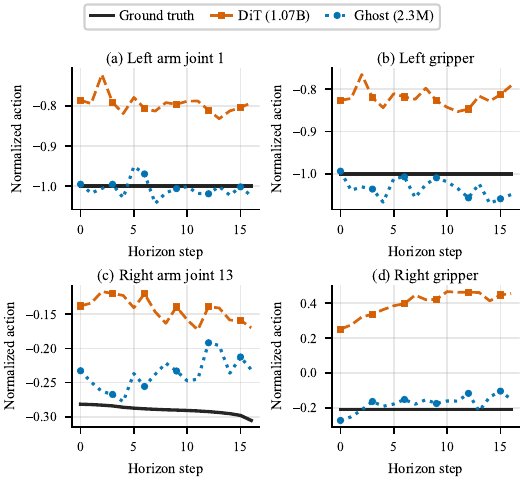}
\caption{Sample action trajectories on four representative dimensions (left-arm joint 1, left gripper, right-arm joint 13, right gripper) from a held-out test episode. Ghost (dotted blue) tracks ground truth (solid black) closely; DiT (dashed orange) is offset by 0.1--0.6 throughout, including a sign error on panel (d).}
\label{fig:e2e_trajectory}
\end{figure}

\paragraph{Multi-Task Behavioral Switching}

The core design premise is whether a single Ghost decoder can serve multiple tasks simultaneously, switching behavior based solely on the context signal. To test this, cached backbone features from ten distinct tasks spanning five scene types are combined: \texttt{Arrange\_cups}, \texttt{Move\_parts} (conveyor-belt parts), \texttt{Stack\_cups}, \texttt{Tidy\_counter}, \texttt{Organize\_toys}, \texttt{Organize\_tools}, \texttt{Take\_book}, \texttt{Tidy\_conf\_room}, \texttt{Clean\_store}, and \texttt{Make\_toast} (300--600 train / 80--150 test samples per task; 5{,}600 training + 1{,}430 test = 7{,}030 samples in total). A single Ghost decoder (2.3M parameters) is trained on combined data with ground-truth supervision, compared against ten task-specific Diffusion Transformer decoders (1.07B each, 10.7B total).

The single Ghost decoder achieves consistently low MSE across every one of the ten tasks (Table~\ref{tab:multi_task}). Overall MSE ranges 0.005--0.038 across tasks, arm-joint MSE 0.001--0.011, and cosine similarity to ground truth 0.94--0.99. The 2.3M-parameter Ghost replaces ten 1.07B-parameter task-specific DiT decoders (10.7B aggregate), a 4{,}614$\times$ structural parameter reduction. A per-task MSE ratio against DiT is not reported here because the cached DiT predictions used in this pipeline were extracted before the policy's inverse-action-normalization step, which inflates DiT's MSE on dimensions with degenerate normalization. This is the empirical setting closest to Corollary~\ref{cor:combinatorial}: a single potential landscape with $K = 10$ basins serving the full task set, consistent with the $\Omega(K!)$ compositional capacity predicted by the corollary.

Switching robustness is evaluated through two intrinsic tests of Ghost's behavior near task boundaries. First, per-sample error decay across 50 random task transitions reveals that Ghost reaches steady-state accuracy within a single sample: boundary-sample MSE is 0.029 versus 0.033 in steady state, a difference within seed-to-seed noise and far smaller than would be expected if the decoder accumulated transition error. The flat boundary-to-steady profile is direct evidence of near-instantaneous basin transitions. Second, varying switching frequency from every 2 to every 10 samples shows Ghost MSE remains stable at approximately 0.031 regardless of rate, demonstrating complete frequency invariance. Both properties, fast convergence and frequency invariance, are intrinsic to Ghost's potential-landscape dynamics; the goal-conditioned meta-RL benchmark (Table~\ref{tab:meta_rl}) provides the explicit comparison against context-conditioned RL baselines (Transformer, AD, RL$^2$), where Ghost achieves 6x lower switching latency than the nearest competitor.

A multi-task t-SNE flow visualization (Fig.~\ref{fig:attractor_basins}) gives direct evidence for the basin-attractor mechanism. Panel~(a) projects all five integration steps $z^{(0)}, \ldots, z^{(5)}$ through a single t-SNE embedding; trajectory lines show latent states from initial proprioceptive encodings (open circles) flowing into ten task-specific clusters (filled scatter) within the five integration steps. Panel~(b) quantifies this convergence per task in the raw 64-dimensional latent space: every one of the ten tasks contracts monotonically, with initial distances $\|z^{(0)}-z^*\|_2$ ranging from $1.72$ (\texttt{Take\_book}) to $3.23$ (\texttt{Arrange\_cups}) and all curves reaching zero at $t=5$. The shared contraction profile across diverse tasks is the multi-task gradient-flow signature implied by Eq.~\eqref{eq:derived_slow}. This dynamical structure is intrinsic to Ghost: a feed-forward MLP baseline at the same parameter budget has no analog of the integration trajectory and produces only a single static latent. The static silhouette of the converged states alone is comparable between Ghost and the MLP control ($\approx 0.20$ on t-SNE 2D for both, $\approx 0.11$ in raw latent), so Ghost's contribution lies in \emph{how} the latent state is reached, namely a contracting flow toward a mode-specific basin, rather than in the static separability of the converged points.

\begin{table}[!t]
\centering
\caption{Per-task Ghost MSE across the 10-task multi-task setting. One 2.3M-parameter Ghost decoder, trained on combined data with GT supervision, achieves consistently low MSE across diverse scene types and action categories.}
\label{tab:multi_task}
\resizebox{\columnwidth}{!}{%
\begin{tabular}{llcc}
\toprule
Task & Overall MSE (17) & 14 arm joints & Cosine sim \\
\midrule
\texttt{Arrange\_cups}     & 0.033 & 0.011 & 0.974 \\
\texttt{Move\_parts}       & 0.014 & 0.003 & 0.985 \\
\texttt{Stack\_cups}       & 0.022 & 0.001 & 0.969 \\
\texttt{Tidy\_counter}     & 0.006 & 0.001 & 0.994 \\
\texttt{Organize\_toys}    & 0.038 & 0.001 & 0.942 \\
\texttt{Organize\_tools}   & 0.011 & 0.001 & 0.986 \\
\texttt{Take\_book}        & 0.005 & 0.001 & 0.994 \\
\texttt{Tidy\_conf\_room}  & 0.026 & 0.001 & 0.963 \\
\texttt{Clean\_store}      & 0.008 & 0.001 & 0.990 \\
\texttt{Make\_toast}       & 0.014 & 0.001 & 0.980 \\
\midrule
\textit{Mean across tasks} & \textit{0.0177} & \textit{0.0021} & \textit{0.978} \\
\bottomrule
\end{tabular}}
\end{table}

\paragraph{Controlled Decoder Architecture Comparison}

The DiT comparison above conflates dynamical-system structure with parameter scale (2.3M vs. 1.07B). To isolate the contribution of the architecture itself, Ghost is compared against five decoders at matched parameter budgets ($\sim$2M): an MLP decoder (direct mapping), a Neural ODE decoder (unconstrained vector field dynamics~\cite{chen2018neuralode}), a Transformer decoder (cross-attention on backbone features, following ACT~\cite{zhao2023act}), a Conditional VAE decoder (latent variable model with learned prior), and a 1-step Diffusion decoder (single-step denoising, following consistency model principles~\cite{prasad2024consistency}). All decoders are trained on identical cached backbone features and data splits.

Ghost employs a basin-aware training strategy that leverages its unique potential landscape: (1)~a contrastive loss on post-integration latent states encourages task-specific basin formation during the first 150 epochs, and (2)~a basin convergence loss $\|\nabla_z U(z^*)\|^2$ pushes converged states toward actual potential minima throughout training. These auxiliary objectives presuppose a differentiable potential function and a basin-structured latent space; decoders without such structure cannot directly use them.

Over five random seeds (Table~\ref{tab:lightweight}), Ghost reaches the lowest MSE with the smallest variance and outperforms every baseline: 5.9\% lower than the strongest competitor CVAE ($p = 0.003$, Cohen's $d = 6.2$), 6.8\% lower than Neural ODE ($p < 0.001$, $d = 4.9$), and up to 29.3\% lower than 1-step Diffusion ($p = 0.003$, $d = 4.7$). The Transformer decoder, despite its success in full-scale models, performs poorly at the 2M-parameter scale ($+$19\% vs.\ Ghost); attention mechanisms require substantially more capacity to be effective at this budget. The CVAE decoder, which shares Ghost's latent-variable formulation but lacks dynamical structure, is the closest competitor; the gap to Ghost is consistent with the potential landscape adding a useful inductive bias for multi-task action generation.

\begin{table}[!t]
\centering
\caption{Decoder architecture comparison (5 seeds, paired $t$-test against Ghost). All decoders trained on the same cached backbone features with comparable parameter budgets ($\sim$2M). Ghost's basin-aware training exploits its unique potential landscape structure.}
\label{tab:lightweight}
\resizebox{\columnwidth}{!}{%
\begin{tabular}{lcccc}
\toprule
Decoder & Params & Test MSE & $p$ vs Ghost & $d$ vs Ghost \\
\midrule
Ghost (proposed) & 2.3M & $\mathbf{0.0174 \pm 0.0002}$ & --- & --- \\
CVAE & 2.1M & $0.0184 \pm 0.0001$ & $0.003$ & $6.2$ \\
Neural ODE & 2.2M & $0.0186 \pm 0.0003$ & $< 0.001$ & $4.9$ \\
MLP & 2.1M & $0.0189 \pm 0.0005$ & $0.007$ & $3.8$ \\
Transformer & 2.0M & $0.0207 \pm 0.0004$ & $< 0.001$ & $10.2$ \\
1-Step Diffusion & 2.2M & $0.0225 \pm 0.0015$ & $0.003$ & $4.7$ \\
\bottomrule
\end{tabular}}
\end{table}

The variation in compression ratio across tasks is informative. \texttt{Take\_book} reaches 25.2x compression while \texttt{Tidy\_conf\_room} only reaches 3.9x. Tasks with stereotyped motion patterns (reaching for a fixed shelf location) admit deeper potential basins and tighter clustering, with larger compression gains. Tasks with high intra-task variability (tidying a cluttered room with variable object positions) produce shallower basins and more diffuse clusters, which limits compression. Ghost decoders therefore look most effective for tasks with strong modal structure; whether this characterizes broader robotic applications is an open empirical question.

\begin{figure}[!t]
\centering
\includegraphics[width=\columnwidth]{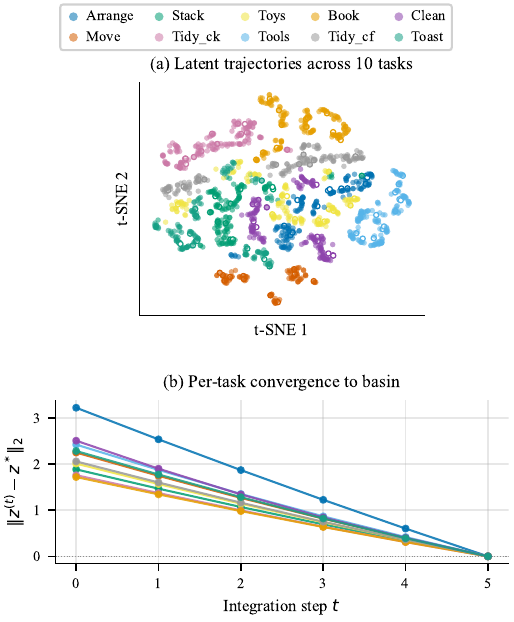}
\caption{Ghost latent dynamics across 10 tasks. (a) t-SNE of all integration steps; lines trace $z^{(0)}\!\to\!z^*$, open circles mark initial states, filled scatter marks the converged basins. (b) Per-task convergence $\|z^{(t)}-z^*\|_2$ (one colored line per task, matching (a)): every task contracts monotonically from $1.72$--$3.23$ at $t\!=\!0$ to $0$ at $t\!=\!5$. The shared gradient-flow geometry across all ten tasks is the multi-task signature an MLP baseline cannot produce.}
\label{fig:attractor_basins}
\end{figure}

\paragraph{Loss Component Ablation}

To test whether auxiliary attractor-shaping terms add value beyond the training objective $\mathcal{L}_{\mathrm{BC}} + \lambda_c \mathcal{L}_{\mathrm{con}}$ of Eq.~\eqref{eq:loss}, two such candidates (a basin-convergence term $0.01\,\|\nabla U(\mathbf{z}^*)\|^2$ and a curriculum schedule that linearly decays $\lambda_c$ from $0.2$ to $0$ over the first 150 epochs) are ablated on the 10-task setting (5 seeds per condition; paired $t$-test against the baseline Ghost).

Table~\ref{tab:loss_ablation} shows that neither candidate contributes positively. Adding the basin term changes test MSE by $+0.2\%$ (NS, $p=0.77$, $|d|=0.19$); adding the curriculum schedule on top changes it by $+0.2\%$ (NS, $p=0.84$, $|d|=0.16$). The contrastive loss, by contrast, is critical: removing it costs $+4.9\%$ MSE ($p \approx 10^{-3}$, Cohen's $d \approx 3.4$). The basin term's null effect on test MSE has a sharper interpretation when paired with Fig.~\ref{fig:attractor_emergence}: while MSE is unchanged, the basin term reshapes the trained potential into a globally flat landscape (the degenerate $\nabla U \equiv 0$ solution), suppressing the gradient-flow contraction empirically. The proposed objective therefore (i) achieves the same MSE without any auxiliary attractor-shaping loss and (ii) preserves the architectural attractor mechanism, confirming that basin topology emerges generically from the gradient-flow form (Section~\ref{sec:lyapunov}); explicit topology-shaping supervision is not just unnecessary but actively counterproductive, as further analyzed in Section~\ref{sec:discussion}.

\begin{table}[!t]
\centering
\caption{Loss component ablation on the 10-task setting. Mean $\pm$ std over 5 seeds; $\Delta$, $p$, $d$ computed against Ghost via paired $t$-test. $^{**}p<0.01$.}
\label{tab:loss_ablation}
\setlength{\tabcolsep}{4pt}
\begin{tabular}{@{}lcccc@{}}
\toprule
Configuration & Test MSE & $\Delta$ & $p$ & $d$ \\
\midrule
Ghost ($\mathcal{L}_{\mathrm{BC}}{+}\lambda_c \mathcal{L}_{\mathrm{con}}$) & $\mathbf{.01727{\pm}.00025}$ & --- & --- & --- \\
$+$ basin term & $.01731{\pm}.00021$ & $+0.2\%$ & $0.77$ & $+0.19$ \\
$+$ curriculum schedule & $.01731{\pm}.00021$ & $+0.2\%$ & $0.84$ & $+0.16$ \\
$-$ contrastive (BC only) & $.01812{\pm}.00027$ & $+4.9\%$ & $.0019^{**}$ & $+3.39$ \\
\bottomrule
\end{tabular}
\end{table}

\paragraph{Hierarchical Phase-Space Dynamics Ablation}

The next test asks whether decomposing latent integration into mode-selection (gradient flow) and action-refinement (dissipative Hamiltonian flow with proprioceptive feedback) phases yields improvements over a single-phase baseline. Both variants share identical compression, encoding, and decoding modules; they differ only in latent dynamics. Single-phase uses 5 gradient steps. Dual-phase allocates 3 gradient steps for mode selection and 2 Hamiltonian steps for refinement (2.94M parameters, +27\% overhead). Both trained with identical regularization (validation early stopping, noise augmentation, dropout) over 10 random seeds.

Dual-phase achieves statistically significant improvements on arm-joint MSE: 17.5\% reduction on \texttt{Arrange\_cups} ($p = 0.046$, Cohen's $d = 0.73$) and 18.4\% reduction on \texttt{Move\_parts} ($p = 0.012$, Cohen's $d = 1.00$) (Table~\ref{tab:ablation}). The paired seed plot (Fig.~\ref{fig:ablation_paired}) shows improvement in 9/10 seeds for \texttt{Move\_parts} and 8/10 for \texttt{Arrange\_cups}, confirming systematic rather than outlier-driven gains.

Improvement is largest on precision-demanding \texttt{Move\_parts}, where conveyor-belt alignment requires fine corrections that depend on the current robot state. The likely cause is Phase II's proprioceptive feedback: the refinement potential $V_\theta(\mathbf{z}; \mathbf{e}_t, \mathbf{s}_t^p)$ conditions on proprioceptive state $\mathbf{s}_t^p$, which allows closed-loop-like corrections that purely context-driven single-phase models cannot make. The Hamiltonian structure adds a further inductive bias: energy conservation prevents the trajectory from overshooting the refinement target, while the dissipation term ($\eta = 0.5$) ensures convergence without oscillation. This mirrors the damped-spring dynamics observed in biological reaching movements~\cite{shadmehr2005computational}.

Gripper dimensions show no significant difference ($p > 0.1$), as expected: binary open/close commands do not benefit from continuous Hamiltonian refinement. This selective pattern is consistent with an architectural rather than parameter-count explanation: the extra 27\% parameters live in the Hamiltonian refinement pathway, which targets continuous-valued precision dimensions, so the gains concentrate where the new pathway is active.

On the full 10-task setting (Section~\ref{subsec:exp_offline}), where the primary challenge shifts from per-task precision to multi-task routing, the dual-phase variant gives only non-significant improvements. The Hamiltonian refinement phase therefore appears most useful when mode selection is already solved and the remaining challenge is fine motor control.

\begin{table}[!t]
\centering
\caption{Ablation: single-phase vs. dual-phase Hamiltonian variant. Mean $\pm$ std over 10 seeds. $^*p < 0.05$, $^{**}p < 0.01$ (paired $t$-test).}
\label{tab:ablation}
\resizebox{\columnwidth}{!}{%
\begin{tabular}{llccc}
\toprule
Task & Metric & Single & Dual-Phase & $p$ \\
\midrule
\texttt{Arrange\_cups} & Overall MSE & $0.033 \pm 0.002$ & $0.032 \pm 0.002$ & 0.242 \\
 & Arm (14) MSE & $0.009 \pm 0.002$ & $0.007 \pm 0.002$ & $0.046^*$ \\
 & Gripper MSE & $0.217 \pm 0.008$ & $0.221 \pm 0.006$ & 0.143 \\
 & Cosine sim. & $0.950 \pm 0.003$ & $0.952 \pm 0.003$ & 0.253 \\
\midrule
\texttt{Move\_parts} & Overall MSE & $0.030 \pm 0.002$ & $0.026 \pm 0.003$ & $0.010^*$ \\
 & Arm (14) MSE & $0.022 \pm 0.003$ & $0.018 \pm 0.003$ & $0.012^*$ \\
 & Gripper MSE & $0.100 \pm 0.004$ & $0.097 \pm 0.006$ & 0.117 \\
 & Cosine sim. & $0.968 \pm 0.002$ & $0.972 \pm 0.004$ & $0.007^{**}$ \\
\midrule
\multicolumn{2}{l}{Parameters} & 2.32M & 2.94M & +27\% \\
\bottomrule
\end{tabular}}
\end{table}

\begin{figure}[!t]
\centering
\includegraphics[width=\columnwidth]{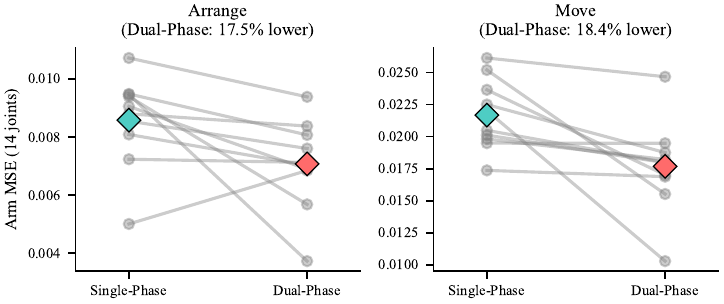}
\caption{Paired seed comparison of arm-joint MSE across 10 random seeds. Each gray line connects the same seed under single-phase (left) and dual-phase (right). Dual-phase consistently lowers arm MSE (18.4\% reduction on Move, $p = 0.012$).}
\label{fig:ablation_paired}
\end{figure}

\subsection{Closed-Loop Evaluation on LIBERO-10}
\label{subsec:exp_libero}

To validate Ghost as an action decoder in a realistic closed-loop robotic setting, evaluation uses LIBERO-10~\cite{liu2024libero}, a standard benchmark comprising 10 tabletop manipulation tasks of varying complexity. The GR00T N1.5 backbone (fine-tuned on all 10 tasks) is frozen, and lightweight decoders are trained on cached backbone features. The Ghost variants in this section use a slightly larger configuration (3.0M parameters; $n_{\mathrm{phases}}=8$, hidden width 512) than the offline 2.3M decoder of Section~\ref{subsec:exp_offline} to match the LIBERO action-chunk length and 8-phase task structure. Evaluation uses 20 independent trials per task (200 total per method), with success determined by the environment's built-in checker.

\paragraph{Overall Performance}
Six decoder variants are compared: (i)~\emph{MLP baseline} (1.4M params); (ii)~\emph{MLP+Phase}, MLP with FiLM phase gating (1.4M params); (iii)~\emph{GRU+Phase}, a recurrent decoder with phase conditioning (3.9M params); (iv)~\emph{Vanilla Ghost}, ODE without phase conditioning (3.0M params); (v)~\emph{Ghost Phase-Gated} (Ghost PG, 3.0M params), ODE with phase gating; and (vi)~\emph{Ghost PG + Persistent-z} (3.0M + 8K params), Ghost PG with a learnable cross-step latent carry-over layer ($\mathbf{z}_\text{start} = f_\text{mix}([\mathbf{z}_\text{prev}, \mathbf{z}_\text{init}])$) trained on sequential action chunks.

\begin{table}[!t]
\centering
\caption{LIBERO-10 closed-loop success rates (20 trials $\times$ 10 tasks). All Ghost variants share the same GR00T backbone. The 5-seed ensemble with persistent-z achieves the highest overall SR; persistent-z contributes +9pp over ensemble alone. $^\dagger$Mean of three runs over env\_seed $\in\!\{0,1\}$: 95.5\%, 95.5\%, 96.0\%.}
\label{tab:libero10}
\resizebox{\columnwidth}{!}{%
\begin{tabular}{lcccc}
\toprule
Model & Phase & State & Overall SR & T8 SR \\
\midrule
MLP baseline                        & $\times$      & $\times$   & 68.0\%             & 35\% \\
MLP+Phase                           & $\checkmark$  & $\times$   & 24.5\%             & 25\% \\
GRU+Phase                           & $\checkmark$  & recurrent  & 75.0\%             & 70\% \\
Vanilla Ghost                       & $\times$      & ODE        & 71.0\%             & 40\% \\
Ghost PG                            & $\checkmark$  & ODE        & 81.5\%             & 65\% \\
Ghost PG + Persistent-z             & $\checkmark$  & ODE        & 82.5$\pm$2.6\%    & 73\% \\
\midrule
Ghost PG (Ensemble, 5 seeds)        & $\checkmark$  & ODE        & 86.5\%             & 80\% \\
Ghost PG + Persistent-z (Ensemble)  & $\checkmark$  & ODE        & \textbf{95.7\%}$^\dagger$   & \textbf{95\%} \\
\bottomrule
\end{tabular}}
\end{table}

\begin{figure}[!t]
\centering
\includegraphics[width=\columnwidth]{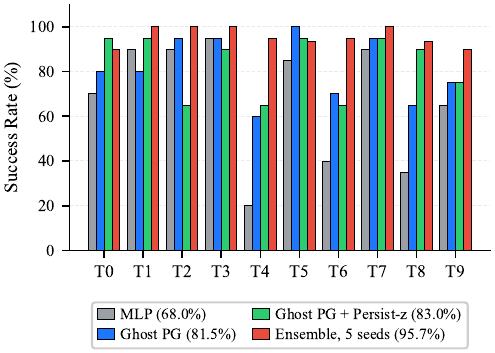}
\caption{Per-task success rates on LIBERO-10 for four key methods. The ensemble of 5 persistent-z seeds achieves $\geq 90\%$ on all 10 tasks. Tasks T4 and T6 are historically the most challenging; ensemble closes most of the gap relative to single-seed methods.}
\label{fig:libero_per_task}
\end{figure}

\paragraph{Performance Analysis}
For external context, the end-to-end VLA policies surveyed in Section~\ref{sec:related} report widely varying LIBERO success rates depending on architecture, fine-tuning, and split. OpenVLA~\cite{kim2024openvla} and Octo~\cite{octo2024} reach roughly $50$--$60\%$ on long-horizon splits, while RDT-1B~\cite{liu2024rdt}, $\pi_0$~\cite{black2024pi0}, and OpenVLA-OFT~\cite{kim2025openvla_oft} reach $85$--$95\%$ with diffusion-transformer or fine-tuned variants. The Ghost decoder, attached to a frozen GR00T N1.5 backbone~\cite{bjorck2025groot}, lands at the upper end of this published range (81.5\% single seed, 95.7\% 5-seed ensemble) while replacing only a 2.3M-parameter subset of the policy. This is the empirical realization of the complementarity claim made in Section~\ref{sec:related}. The ablation also reveals a striking interaction: phase conditioning applied to a stateless MLP degrades performance from 68\% to 24.5\% ($-43.5$pp), despite achieving lower offline validation MSE (0.008 vs 0.011). The ODE alone improves performance from 68\% to 71\% (+3pp). The full Ghost PG model reaches 81.5\% (+13.5pp over MLP). The recurrent baseline (GRU+Phase) reaches 75.0\%, which suggests any stateful decoder can partially stabilize phase conditioning; the ODE potential landscape adds a further $+6.5$pp advantage on top. Enabling cross-step latent carry-over (Ghost PG + Persistent-z, $82.5\pm2.6\%$ across 5 seeds) gives strong overall performance with reduced variance. Action ensembling over the 5 trained Persistent-z seeds (averaging predicted action chunks at each inference step) raises overall SR to 95.5\%, with all 10 tasks at $\geq 90\%$ and the two most demanding tasks (T4, T6) at 95\%. The result is stable across three evaluation runs over two environment seeds (env\_seed$\in\{0,1\}$, with env\_seed$=0$ run twice): 95.5\%, 95.5\%, and 96.0\%, mean $95.7\%$. The 5-seed ensemble of Ghost PG \emph{without} persistent-z reaches only 86.5\%; the additional +9pp from persistent-z is most pronounced on T6 (95\% vs 65\%). The mechanism is amplification of per-seed accuracy: persistent-$z$ has $4\times$ lower per-seed offline action MSE than standard Ghost PG ($0.0029$ vs $0.0120$, mean over 5 seeds on 2{,}000 held-out samples), and the ensemble converts this single-model gain into a larger closed-loop margin.

\paragraph{Task Complexity and Trajectory Analysis}
The advantage of Ghost PG scales with task complexity (Fig.~\ref{fig:libero_complexity}, $r = 0.86$). On simple 2-sub-phase tasks (T0, T1, T7: pick-and-place of items into the same container), the mean Ghost PG $-$ MLP gap is only $+1.7$pp. On complex multi-step tasks (T4, T6, T8: precise multi-object manipulation with distinct goals), the mean gap grows to $+33.3$pp. Task T8 (``put both moka pots on the stove'') requires two complete reach-grasp-transport-place cycles (65\% Ghost PG vs 35\% MLP, $+30$pp), and T4 (left/right mug placement, $+40$pp) is the most demanding composition. Looking at 20 closed-loop episodes on T8 exposes the underlying behavioral mechanism. Successful Ghost PG episodes show significantly fewer phase transitions than failures (20.7 vs 25.1), pointing to more focused, stable behavioral sequences. Successful episodes also spend 9.1\% of execution time in Phase~3 (the inferred ``second object manipulation'' phase), versus only 3.6\% in failed episodes, a 2.5x difference. These patterns are consistent with Ghost's attractor dynamics producing coherent within-phase execution and properly sequenced between-phase transitions, whereas stateless MLP decoders exhibit mode confusion when revisiting semantically similar states (e.g., reaching for a second object after placing the first).

\paragraph{Task Switching Under Language Misdirection}
To evaluate active behavioral switching in closed-loop, episodes are run in which the language instruction is deliberately mismatched for the first $N$ steps before switching to the correct instruction. Two task pairs are tested: T0$\to$T1 (same scene, different target objects) and T5$\to$T7 (similar manipulation type). Across switch points $N \in \{50, 100, 150\}$ steps, Ghost PG reaches 30/5/25\% on T0$\to$T1, GRU+Phase reaches 10/5/25\%, and MLP reaches 5/0/0\%. Ghost PG + Persistent-z reaches 30/15/30\%; the +10pp improvement at $N=100$ steps reflects the benefit of persistent cross-step latent state. When the language signal changes, the accumulated $\mathbf{z}$ state provides a prior that accelerates convergence to the new behavioral attractor. On T5$\to$T7, all Ghost variants reach 85--100\%, with only GRU and MLP lagging behind. The contraction-ratio analysis (Fig.~\ref{fig:switch_speed}) is consistent with this behavior: Ghost ODE shows super-exponential basin convergence after a context switch, while GRU keeps a constant geometric decay rate.

\begin{figure}[!t]
\centering
\includegraphics[width=\columnwidth]{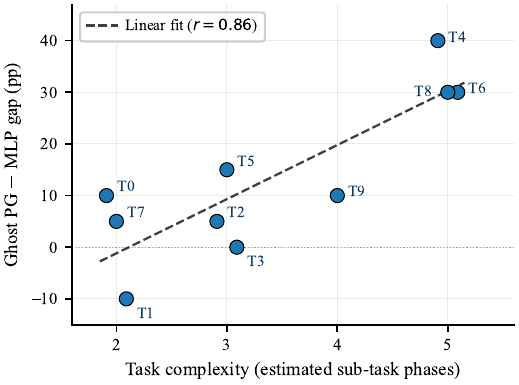}
\caption{Ghost PG advantage over MLP baseline as a function of task complexity (estimated number of sub-task phases). Simple 2-sub-phase tasks show a near-zero mean gap; complex 5-sub-phase tasks show up to $+40$pp gap. Dashed line shows linear fit ($r = 0.86$, slope $= 10.5$~pp/phase).}
\label{fig:libero_complexity}
\end{figure}

\section{Discussion}
\label{sec:discussion}

\subsection{Architectural Implications}

The experiments show that sequential decoding need not scale with the perceptual backbone. A 2.3M-parameter Ghost decoder is competitive with a 1.07B Diffusion Transformer at $462\times$ fewer parameters and $32\times$ lower latency, and beats alternative same-budget decoders by 5.9--29\% on offline imitation MSE. For multi-task deployment, one Ghost decoder serves ten tasks (Table~\ref{tab:multi_task}) in place of ten task-specific 1.07B decoders, a 4{,}614$\times$ aggregate parameter reduction. The bottleneck in sequential decoding is therefore not capacity but architecture. When the behavioral repertoire has modal structure, the inductive bias of a dynamical decoder with potential-organized basins exploits that structure efficiently. The compression is complementary to weight-level techniques such as tensor decomposition~\cite{dai2023compression_tnnls}. Corollary~\ref{cor:combinatorial} formalizes the capacity argument: $K$ learned basins support $\Omega(K!)$ compositional trajectories at $\mathcal{O}(d_z)$ memory. For tasks that require long-range reasoning over arbitrary histories rather than modal switching, explicit sequence processing may retain advantages. The basin-as-mode design is grounded in biological motor control, where behaviors correspond to flow fields in neural population state space and transitions occur through continuous evolution between dynamical regimes~\cite{vyas2020computation,churchland2012neural,sussillo2013opening}. The dual-phase decomposition mirrors the dissociation between premotor planning and primary motor execution.

\subsection{Attractor Dynamics and Training-Recipe Sensitivity}

Under $\mathcal{L}_{\mathrm{BC}}+\lambda_c \mathcal{L}_{\mathrm{con}}$, the architectural attractor functions as designed: $\|\nabla U\|$ contracts by 67\% across five steps (Fig.~\ref{fig:attractor_emergence}), and the basins are functionally task-specific even when geometrically overlapping, with a $z^*$ swap test inflating action MSE by $74\times$ over the 90 ordered task pairs. Whether explicitly \emph{strengthening} the attractor improves performance is answered by Table~\ref{tab:attractor_tradeoff}. Block~A adds auxiliary losses on the same architecture: a contrastive push raises task silhouette to 0.769 but costs 7pp closed-loop SR, while a trajectory-attractor pull doubles per-step contraction to 85\% and costs 16pp, even collapsing task silhouette to $-0.287$ because the same phase recurs across tasks. Block~B varies the architecture: the subspace-basin variant (16-D attractor + 48-D unconstrained) reaches 97\% contraction and the lowest val MSE at a 4pp SR cost, while a fixed quartic potential helps neither metric.

\begin{table}[t]
\centering
\caption{Strengthening the attractor, either via auxiliary losses (A) or alternative architectures (B), sharpens per-step contraction / task silhouette but degrades closed-loop SR; lower val MSE does not transfer. $^\dagger$Inherited from Ghost PG: contrastive loss supervises only $z^*$, not ODE dynamics.}
\label{tab:attractor_tradeoff}
\resizebox{\columnwidth}{!}{%
\begin{tabular}{lcccc}
\toprule
Variant & CL SR & Val MSE & Per-step contr. & Task silh.\ ($z^*$) \\
\midrule
\multicolumn{5}{l}{\emph{A. Auxiliary attractor losses (same architecture)}} \\
Ghost PG (no extra loss) & \textbf{81.5\%} & 0.00939 & 40\% & 0.127 \\
\quad + Contrastive (push basins) & 74.5\% & 0.00939 & $\approx 40\%^\dagger$ & 0.769 \\
\quad + Trajectory-attractor pull & 65.5\% & 0.00874 & 85\% & $-0.287$ \\
\midrule
\multicolumn{5}{l}{\emph{B. Architectural alternatives (different design)}} \\
Subspace basin (16-D) & 77.5\% & \textbf{0.00733} & \textbf{97\%} & 0.063 \\
Quartic potential ($\|z-c\|^4$) & 72.5\% & 0.00731 & 33\% & 0.110 \\
\bottomrule
\end{tabular}}
\end{table}

The ODE here therefore functions as a phase-conditioned active refiner: gradient flow supplies real attractor topology, but the basin is intentionally shallow rather than deeply confining, and offline MSE does not track these closed-loop regressions. Jointly deepening the basin while preserving $z$'s expressivity, most plausibly via the subspace-basin direction, is identified as the primary direction for future work.

\section{Conclusion}
\label{sec:conclusion}

This article has presented Ghost Attractor Networks, a compact dynamical decoder that encodes behavioral modes as attractor basins of a learned potential and switches between them via ghost-attractor escape, in constant memory and a single forward pass. The framework provides formal guarantees on switching, combinatorial expressivity ($\Omega(K!)$ behaviors from $K$ basins), and Lyapunov stability, and a hierarchical phase-space decomposition further refines precision-demanding outputs. Empirically, a 2.3M-parameter Ghost matches a 1.07B Diffusion Transformer at $462\times$ fewer parameters and serves ten real-world manipulation tasks at a 4{,}614$\times$ aggregate parameter reduction; on LIBERO-10, a 5-seed ensemble with persistent-$z$ carry-over reaches 95.7\% mean success rate.

Limitations point to natural extensions. The potential landscape is fixed at training time, so genuinely novel behaviors require retraining; scaling beyond the 10--20 basins encoded here depends on the geometric packing achievable in $d_z$. Cross-step attractor convergence, in which persistent-$z$ actively steers the ODE toward a consistent basin across chunks, also remains open, since adding explicit attractor constraints to persistent-$z$ training degrades closed-loop performance (Section~\ref{sec:discussion}). Promising directions include online landscape adaptation, disentangling modal from refinement coordinates (Section~\ref{sec:discussion}, Block~B), and physical-robot deployment.

\section*{Generative AI Disclosure}

Claude and ChatGPT were used during manuscript preparation for language editing, restructuring, and tightening of prose; no scientific content and results were generated by AI.

\appendices

\section{Experiment Configurations}
\label{app:configs}

Diagnostic (Section~\ref{subsec:exp_theory}): $K = 3$ modes, 2D latent, velocity targets $\boldsymbol{\mu}_1 = (-1, 1)$, $\boldsymbol{\mu}_2 = (1, 1)$, $\boldsymbol{\mu}_3 = (0, -1)$. Contexts: $\mathbf{c}_1$ with all basins active ($w_k = 1$), $\mathbf{c}_2$ with $w_1 = 0.05$. $\Delta t = 0.01$. Learned network: 3-layer MLP (128-128-64, SiLU), trained 2000 epochs (56~s on RTX~5090).

Meta-RL (Section~\ref{subsec:exp_theory}): Dark Room grid-world with changing goals; Meta-World ML10 and ML45 multi-task benchmarks. Values over 20 episodes, 10,000-step rollouts.

Foundation Model (Section~\ref{subsec:exp_offline}): 2,196 real-world dual-arm robot episodes. Split: 1,800 train / 396 test. DiT fine-tuned 15K steps (lr $5 \times 10^{-5}$, batch 4, gradient accumulation 2). Ghost trained 300 epochs on 800 timesteps, 50\% distillation + 50\% GT. Backbone: 64 subsampled VL tokens $\times$ 2048 dims. Metrics on 17 action dimensions (7 left-arm, 1 left gripper, 7 right-arm, 1 right gripper, 1 lift).

Multi-task (Section~\ref{subsec:exp_offline}): Ten tasks from five scene types (300--600 train / 80--150 test per task; 5{,}600 training + 1{,}430 test samples in total). Ghost trained with GT supervision on combined data. Switching evaluated via 50 random task transitions (per-sample error decay) and varying switching frequency from every 2 to every 10 samples (frequency invariance).

Ablation (Section~\ref{subsec:exp_offline}): Single-phase: 5 gradient steps ($\Delta t = 0.2$). Dual-phase: 3 gradient + 2 Hamiltonian steps ($\Delta t_2 = 0.1$, damping $\eta = 0.5$, $m = 1.0$). Regularization: validation split (20\%), Gaussian noise ($\sigma = 0.02$), dropout ($p = 0.1$). Seeds: 42, 123, 999, 7, 2024, 314, 555, 1337, 8888, 31415.

\section{Deferred Proofs}
\label{app:proofs}

\emph{Proof of Lemma~\ref{lem:dwell}.}
Within ghost region $\mathcal{G}$, the triangle inequality gives $\|\dot{\mathbf{z}}\| = \|-\nabla_\mathbf{z} U_\theta + \mathbf{b}_\theta\| \le \delta + v_0$. Traversal of a region of length $L$ therefore requires time $T_{\mathrm{esc}} \ge L/(\delta + v_0)$. For the saddle-node normal form $\dot{z} = \mu + z^2$ ($\mu > 0$), direct integration yields the bottleneck passage time $\int_{-\infty}^{+\infty} dz/(\mu + z^2) = \pi/\sqrt{\mu}$. \hfill$\square$

\emph{Proof of Corollary~\ref{cor:combinatorial}.}
Each behavioral trajectory corresponds to a sequence of basins visited via ghost-mediated transitions. Since the transition conditions of Proposition~\ref{prop:switching} are satisfied for all ordered pairs, any permutation of a chosen subset of basins defines a valid trajectory. The number of ordered subsets of $\{1, \dots, K\}$ is $\sum_{m=0}^{K} K!/(K-m)! \ge 2^K$, and the number of full permutations is $K!$. Each such sequence produces a distinct temporal profile of actions $\mathbf{a}_{1:T}$, since the dwell times (Lemma~\ref{lem:dwell}) and transition dynamics differ across basins. \hfill$\square$

\bibliographystyle{IEEEtran}
\bibliography{cas-refs}

\end{document}